\documentclass{article}

\usepackage{microtype}
\usepackage{graphicx}
\usepackage{subfigure}

\usepackage{comment}
\usepackage{enumitem}
\usepackage[table]{xcolor}

\usepackage{hyperref}
\usepackage[title, toc,page]{appendix}
\usepackage{titlesec}
\usepackage{titletoc}
\usepackage{booktabs} 

\usepackage{braket}
\usepackage{multirow}
\usepackage{natbib}
\usepackage{graphicx}
\usepackage{subfigure}
\usepackage{makecell}
\usepackage{booktabs}
\usepackage{array}
\usepackage{url}
\usepackage{algorithm}
\usepackage{algorithmic}
\usepackage{dsfont}

\usepackage{icml2025}

\usepackage{amsmath}
\usepackage{amssymb}
\usepackage{mathtools}
\usepackage{amsthm}
\usepackage{amssymb}

\usepackage{amsfonts}
\usepackage{bm}
\usepackage{bbm}

\usepackage[capitalize,noabbrev]{cleveref}

\theoremstyle{plain}
\newtheorem{theorem}{Theorem}[section]
\newtheorem{proposition}[theorem]{Proposition}
\newtheorem{lemma}[theorem]{Lemma}
\newtheorem{corollary}[theorem]{Corollary}
\theoremstyle{definition}
\newtheorem{definition}[theorem]{Definition}
\newtheorem{assumption}[theorem]{Assumption}
\theoremstyle{remark}
\newtheorem{remark}[theorem]{Remark}
\usepackage{titlesec}
\usepackage{booktabs}

\usepackage[textsize=tiny]{todonotes}
\usepackage{nicefrac}

\newcommand{\bfa}[1]{\boldsymbol{#1}}

\newcommand{\norm}[1]{\left\lvert \left\lvert#1\right\rvert \right\rvert}
\newcommand{\vertiii}[1]{{\left\vert\kern-0.25ex\left\vert\kern-0.25ex\left\vert #1 
    \right\vert\kern-0.25ex\right\vert\kern-0.25ex\right\vert}}

\newcommand{\pb}{\mathbf{p}}

\newcommand{\yb}{\mathbf{y}}

\newcommand{\ba}{\bm{a}}
\newcommand{\bb}{\bm{b}}

\newcommand{\be}{\bm{e}}

\newcommand{\bh}{\bm{h}}

\newcommand{\bp}{\bm{p}}
\newcommand{\bq}{\bm{q}}

\newcommand{\bu}{\bm{u}}
\newcommand{\bv}{\bm{v}}
\newcommand{\bw}{\bm{w}}
\newcommand{\bx}{\bm{x}}
\newcommand{\by}{\bm{y}}
\newcommand{\bz}{\bm{z}}

\newcommand{\Kb}{\mathbf{K}}

\newcommand{\Qb}{\mathbf{Q}}

\newcommand{\Ub}{\mathbf{U}}
\newcommand{\Vb}{\mathbf{V}}

\newcommand{\Xb}{\mathbf{X}}

\newcommand{\bA}{\bm{A}}

\newcommand{\bH}{\bm{H}}
\newcommand{\bI}{\bm{I}}

\newcommand{\bK}{\bm{K}}

\newcommand{\bQ}{\bm{Q}}

\newcommand{\bU}{\bm{U}}
\newcommand{\bV}{\bm{V}}
\newcommand{\bW}{\bm{W}}
\newcommand{\bX}{\bm{X}}
\newcommand{\bY}{\bm{Y}}
\newcommand{\bZ}{\bm{Z}}

\newcommand{\cD}{\mathcal{D}}

\newcommand{\cF}{\mathcal{F}}

\newcommand{\cH}{\mathcal{H}}

\newcommand{\cW}{\mathcal{W}}
\newcommand{\cX}{\mathcal{X}}

\newcommand{\cZ}{\mathcal{Z}}

\newcommand{\R}{\mathbb{R}}
\newcommand{\N}{\mathbb{N}}
\newcommand{\argmin}{\mathop{\mathrm{argmin}}}

\icmltitlerunning{Transformers and Their Roles as Time Series Foundation Models}

\begin{document}


\twocolumn[
\icmltitle{Transformers and Their Roles as Time Series Foundation Models
}




\icmlsetsymbol{equal}{*}

\begin{icmlauthorlist}
\icmlauthor{Firstname1 Lastname1}{equal,yyy}
\icmlauthor{Firstname2 Lastname2}{equal,yyy,comp}
\icmlauthor{Firstname3 Lastname3}{comp}
\icmlauthor{Firstname4 Lastname4}{sch}
\icmlauthor{Firstname5 Lastname5}{yyy}
\icmlauthor{Firstname6 Lastname6}{sch,yyy,comp}
\icmlauthor{Firstname7 Lastname7}{comp}
\icmlauthor{Firstname8 Lastname8}{sch}
\icmlauthor{Firstname8 Lastname8}{yyy,comp}
\end{icmlauthorlist}

\icmlaffiliation{yyy}{Department of XXX, University of YYY, Location, Country}
\icmlaffiliation{comp}{Company Name, Location, Country}
\icmlaffiliation{sch}{School of ZZZ, Institute of WWW, Location, Country}

\icmlcorrespondingauthor{Firstname1 Lastname1}{first1.last1@xxx.edu}
\icmlcorrespondingauthor{Firstname2 Lastname2}{first2.last2@www.uk}

\icmlkeywords{Machine Learning, ICML}

\vskip 0.3in
]



\printAffiliationsAndNotice{\icmlEqualContribution} 

\begin{abstract}

We give a comprehensive analysis of transformers as time series foundation models, focusing on their approximation and generalization capabilities. 
First, we demonstrate that there exist transformers that fit an autoregressive model on input univariate time series via gradient descent. 
We then analyze MOIRAI \cite{woo2024unified}, a multivariate time series foundation model capable of handling an arbitrary number of covariates. 
We prove that it is capable of automatically fitting autoregressive models with an arbitrary number of covariates, offering insights into its design and empirical success.   
For generalization, we establish bounds for pretraining when the data satisfies Dobrushin’s condition. 
Experiments support our theoretical findings, highlighting the efficacy of transformers as time series foundation models.

\end{abstract}

\section{Introduction}

The advancement of foundation models is reshaping the field of time series forecasting.
Recent studies demonstrate the empirical success of transformer-based time series foundation models \cite{woo2024unified, ansari2024chronos, liang2024foundation, das2023decoder}.
However, a theoretical understanding of how these models succeed is yet missing.
In this paper, we aim to provide a comprehensive analysis of time series foundation models, with a focus on transformer-based models.
We are interested in how these models achieve the \emph{one-model-for-all-datasets} paradigm in time series forecasting.
Specifically, our results cover both uni-variate \cite{ansari2024chronos, das2023decoder, rasul2023lag}, and multi-variate time series foundation models \cite{woo2024unified}\footnote{The model proposed by \cite{woo2024unified} is compatible with an arbitrary number of covariates}.

Our main discovery is twofold.
First, to address universality, we prove that there exists a transformer that fits an autoregressive model \cite{hamilton2020time, mills1990time} on any given uni-variate time series.
Furthermore, we show that the special design of MOIRAI allows transformers to further handle arbitrary number of covariates, making it process any dimension of time series in a principled way.
Second, to address learnability, we establish a generalization bound for pretraining when the data satisfies Dobrushin's condition \cite{Dobrushin1968TheDO, dobrushin1987completely}.
We refer to these two aspects as approximation and generalization, respectively, throughout the rest of this paper, as they form the theoretical foundation for the success of these models.


Our approximation result is inspired by recent studies on in-context learning \cite{bai2024transformers, von2023transformers, li2023transformers, mahankali2023one, ahn2024transformers, zhang2024trained}. 
In-context learning refers to the ability of large foundation models to make accurate predictions for unseen tasks by observing training examples without parameter updates \cite{brown2020language, garg2022can}.
This capability explains how time series foundation models achieve their universality.
By treating the input time series as in-context examples, we show that transformers are able to implement gradient descent to estimate the parameters of the autoregressive model that best explains the given input.

The contribution of this paper is threefold:
\begin{itemize}
    \item 
    From an algorithmic approximation perspective, we prove the existence of a transformer capable of fitting an autoregressive ($\mathtt{AR}$) model on any given uni-variate time series via gradient descent. 
    Extending this to the multi-variate setting, we show that a MOIRAI transformer can automatically adjust the dimensionality of the $\mathtt{AR}$ model to fit time series with an arbitrary number of covariates. 
    Our approximation results not only explain the strong performance of modern models across diverse datasets but also justify the design of MOIRAI.

    \item
    We present the first pretraining generalization bound for time series foundation models.
    We show that when the pretraining data satisfies Dobrushin's condition, the test error can be effectively bounded even when the data does not satisfy the i.i.d. assumption.
    Specifically, when pretraining MOIRAI on $n$ multi-variate time series, the test error decays by a factor of $\nicefrac{1}{\sqrt{n}}$.

    \item 
    Our experimental results match our theories by showing that the prediction error of transformers reduces as the input time series length increases, corresponding to our approximation result.
    
\end{itemize}

\paragraph{Organization.}
The organization of this paper is as follows:
Section 2 describes the problem setup; Section 3 provides the approximation result; Section 4 analyzes the generalization bound for pretraining; Section 5 reports the numerical simulations; and we leave discussions to Section 6.
\paragraph{Notations.}
We use the following notation conventions. 
The vector-valued variable is given by boldfaced characters. 
We denote $[n]:=\{1,\ldots,n\}$ and $[i:j]:=\{i,i+1,\ldots, j\}$ for $i<j$. 
The universal constants are given by $C$ and are ad hoc. 
Considering a sequence of vectors $(\bx_1, \cdots, \bx_T)$, we use $\bx$ without index to represent the whole sequence, and $\bx_{i:j}$ represents $(\bx_i, \cdots, \bx_j)$ for $i < j$.
We impose periodic boundary conditions for the negative index, i.e., $\bx_{-1} = \bx_{T}$.
For a vector $\bfa v$ we denote $\Vert\bfa v\Vert_2$ as its $L_2$ norm. 
For a matrix $\bfa A\in  \R^{m\times n}$ we denote its operator norm as $\Vert\bfa A\Vert_2:=\sup_{\bfa v\in \mathbb{S}^{n-1}}\Vert\bfa A\bfa v\Vert_2$. 
Random variables are given by calligraphic characters $\cX$, and elements from a domain set are given by normal font $\text{x}$.
For more details, see Table~\ref{tab:nomenclature}.

\section{Problem Setup}
This section describes our problem setup.
We introduce the architecture of transformer-based time series foundation models and how we construct our datasets.
\subsection{Transformers}
We consider a sequence of $N$ input vectors $\{ h_i \}_{i=1}^N \subset \R^D$, where $ \bfa H \coloneqq \left[ h_1, \cdots, h_N \right] \in \R^{D \times N}$.
Given any $ \bfa H \in \R^{D \times N}$, we define the attention layer as follows.
\begin{definition}[Attention layer\normalfont]\label{def:attn}
    {\normalfont
    A self-attention layer with $M$ heads is denoted as $\text{Attn}^{\dagger}_{\bm{\theta}_0}(\cdot)$ with parameters $\bm{\theta}_0 = \{ (\bfa V_m), (\bfa Q_m), (\bfa K_m)  \}_{m \in [M]} \subset \R^{D \times D}$.
    The self-attention layer processes any given input sequence $\bfa H \in \R^{D \times N}$ as
    }
    \begin{align*}
    \text{Attn}_{\bm{\theta}_0}^{\dagger} 
    &\left( \bfa H \right)
    \coloneqq
    \bfa H + 
    \frac{1}{N}
    \sum_{m=1}^M
    (\bfa V_m \bfa H)
    \cdot
    \\
    &\sigma 
    \left( 
    \left( \bfa Q_m \bfa H \right)^\top 
    \left(\bfa K_m \bfa H \right) 
    \right),  
\end{align*}
where $\sigma \coloneqq t \mapsto \text{ReLU}(t)/N$.
\end{definition}
\paragraph{Any-variate Attention.}
Next, we introduce the any-variate attention, where \cite{woo2024unified} uses it to replace the standard attention in transformers.
The any-variate attention introduces two learnable variables: Attention Bias $u_1, u_2 \in \R$,
for disambiguation between variates.
\begin{definition}[Any-variate Attention.]\label{def:any-variate-attn}
    An any-variate attention layer with $M$ heads is denoted as $\text{Attn}_{\bm{\theta}_1}(\cdot)$ with parameters $\bm{\theta}_{1} = \{ (\bfa V_m), (\bfa Q_m), (\bfa K_m), (u_m^1), (u_m^2)  \}_{m \in [M]}$.
    With any input $H \in \R^{D \times N}$, we have
    \begin{align*}
    \text{Attn}_{\bm{\theta}_1} &\left( \bfa H \right)
    \coloneqq
    \bfa H + 
    \frac{1}{N}
    \sum_{m=1}^M
    (\bfa V_m \bfa H)
    \times
    \\
    &\sigma 
    \left( 
    \left( \bfa Q_m \bfa H \right)^\top 
    \left(\bfa K_m \bfa H \right) 
    +
    u_m^1 * \bU
    +
    u_m^2 * \bar{\bU}
    \right),  
\end{align*}
where $\sigma \coloneqq t \mapsto \text{ReLU}(t)/N$, $\bU \in \R^{N \times N}$ is a block diagonal matrix with block size $T\in \N^+$, such that each block consists of $1$s, $\bar{\bU} = \bI - \bU$, and $*$ denotes a constant multiply to all entries of a matrix.
\end{definition}

\begin{remark}
    In \cite{woo2024unified}, the attention score is calculated with the RoPE embedding \cite{su2024roformer}:
    \[
    \sigma 
    \left( 
    \left( \bfa Q_m \bfa H \right)^\top 
    \bfa R
    \left(\bfa K_m \bfa H \right) 
    +
    u_m^1 * \bU
    +
    u_m^2 * \bar{\bU}
    \right).  
    \]
    We omit the notation of rotary matrix $\bfa R$ as it is not learnable and is invertible and thus merged into $\bfa Q, \bfa K$ in our analysis.
\end{remark}

\begin{definition}[MLP Layer\normalfont]
    {\normalfont
    We denote an MLP layer with hidden state dimension $D^\prime$ as $\text{MLP}_{\bm{\theta}}(\cdot)$ with parameters $\bm{\theta}_2 = ( \bW_1, \bW_2 ) \in \R^{ D^\prime \times D } \times \R^{D \times D^\prime}$.
    The MLP layer processes any given input sequence $\bfa H \in \R^{D \times N}$ as
    }
    \begin{equation*}
        \text{MLP}_{\bm{\theta}_2} (\bfa H)
        \coloneqq
        \bfa H +
        \bW_2 \sigma(\bW_1 \bfa H).
    \end{equation*}
\end{definition}
Finally, we define a transformer with $L \geq 1$ layers, each consisting of any-variate attention and an MLP layer.
\begin{definition}[MOIRAI Transformer\normalfont]\label{def:moirai}
    {
    \normalfont
    We define the $L$-layer MOIRAI transformer \cite{woo2024unified}, $\text{TF}_{\bm{\theta}}(\cdot)$, as
    }
    \begin{equation*}
        \text{TF}_{\bm{\theta}}(\bfa H)
        =
        \text{MLP}_{\bm{\theta}_2^L}
        \left(
        \text{Attn}_{\bm{\theta}_1^L}
        \left(
        \cdot \cdot
        \text{MLP}_{\bm{\theta}_2^1}
        \left(
        \text{Attn}_{\bm{\theta}_1^1}
        (\bfa H)
        \right)
        \right)
        \right)
        .
    \end{equation*}
    Note that this transformer is equipped with any-variate attention instead of the standard attention.
    For transformers with standard attention, we denote it as 
    $\text{TF}_{\bm{\theta}}^{\dagger}(\cdot)$.
\end{definition}

We use $\bm{\theta}$ to denote the vectorization of all parameters in a transformer and super-index $\ell$ to denote the parameter of the $\ell$-th layer.
Thus, the parameter of a transformer is defined by
\begin{equation*}
        \bfa \theta = 
        \left\{
        \left\{
        \left(\{\bfa Q_m^\ell,\bfa K_m^\ell,\bfa V_m^\ell, u_m^{1, \ell}, u_m^{2, \ell}\}_{m\in[M]}, \bfa W_{1}^\ell,\bfa W_{2}^\ell
        \right)
        \right\}_{\ell\in[L]}
        \right\}.
\end{equation*}    

We denote the ``attention-only" transformers with $\bW_1^{(\ell)}, \bW_2^{(\ell)} = 0$, as $\text{TF}_{\bm{\theta}}^0(\cdot)$ for shorthand.
We define the following norm of a MOIRAI transformer as
\begin{align*}
    \Vert\bfa\theta\Vert_{op}
    \coloneqq
    \max_{\ell\in[L]}
    &\Big\{\max_{m\in[M^{\ell}]}
    \Big\{\Vert \bfa Q_m^\ell\Vert_2,\Vert\bfa K_m^\ell\Vert_2,
        \\
        \vert u_m^{1\ell} \vert , \vert u_m^{2\ell} \vert 
    \Big\}
    &+
    \sum_{m=1}^{M^{\ell}}\Vert\bfa V_m^\ell\Vert_2+\Vert\bfa W_1^{\ell}\Vert_2+\Vert\bfa W_2^{\ell}\Vert_2 \Big\}, 
\end{align*}
where $M^{\ell}$ is the number of heads of the $\ell$-th Attention layer.

\subsection{Data Generation}
Here, we first consider the case where we aim to find a multi-layered transformer that performs least squares regression via In-context learning (ICL).
Specifically, we assume our data is generated from an autoregressive process $\mathtt{AR}_d(q)$ as follows, where $q, d$ denotes the steps of lag and number of covariates, respectively. 
Consider a sequence of data 
$\bx \in \R^{d \times T} \coloneqq (\bx_1, \dots, \bx_T)$, where $\bx_t = (x_t^1, \cdots, x_t^d) \in \R^d$.
Assuming our target (variate of interest) is in dimension $1$, we assume the $\mathtt{AR}_d(q)$ process generates $x_t^1$ as follows:
\begin{equation}\label{eqn:AR-data}
    x_{t}^1
    =
    \sum_{i=1}^q
    \sum_{j=1}^d
    a_i^j \cdot x_{t-i}^j
    + \epsilon_t
    =
    \sum_{j=1}^d
    \langle \bw^j , \bx_{t-q: t-1}^j
    \rangle
    +
    \epsilon_t
    ,
\end{equation}
where $\epsilon_t \sim N(0, 1)$, $a_i^j \in \R^1$.
We denote the concatenation of all weights $\bw^\star = (\bw_1, \cdots, \bw^j)\in \R^{qd}$.
We assume bounded features $\norm{\bx_{t-q:t-1}}_2 \leq B_x$ , for all $t = 1,\cdots, T$.
The first equation writes the $\mathtt{AR}$ process in scalar form, and the second writes it in vector form.
In the following chapters, we will start by considering the uni-variate case ($\mathtt{AR}_1(q)$) and then move on to the multi-variate case ($\mathtt{AR}_d(q)$).
\vspace{-1.5em}
\paragraph{Problem Setup.}
Given a function class $\cF: \R^{d\times T}\mapsto \R$, our goal is to find a universal function $f\in\cF$ such that, given any time series generated from any arbitrary $\mathtt{AR}_d(q)$, its prediction error is bounded by some $\varepsilon \geq 0$, i.e.,
\[
\norm{f( \Tilde{\bx}) - x_{T}^1}_2
\leq
\varepsilon,
\]
where $\Tilde{\bx}$ denotes the time series $\bx$ with $x_T^1$ being masked.
\begin{remark}
    In the appendix, we show that even when the autoregressive process follows some non-linear relationship, there still exists a universal $f$ that predicts all non-linear AR process accurately.
\end{remark}

\section{Approximation}
We study the algorithmic approximation perspective of transformer-based time series foundation models.
We first investigate transformers as uni-variate time series foundation models as a warm-up.
Next, we will move on to MOIRAI \cite{woo2024unified} and analyze how its unique design and pre-processing methods enable its universality.
\subsection{Warm Up: Autoregressive Regression}
We start our analysis with a warm-up example on the $\mathtt{AR}_1(q)$ model.
We show that standard transformers are capable of performing gradient descent via in-context learning on autoregressive data.
Here, we consider an input sequence with the following form
\begin{align}\label{eqn:input-data}
    \bfa H 
    &\coloneqq
    \begin{bmatrix}
        x_1 & x_2 &  \dots & x_{T} & 0
        \\
        \bp_1 & \bp_2 & \dots & \bp_{T} &
        \bp_{T+1}
    \end{bmatrix}
    \in \R^{D \times (T+1)}
    ,
    \\
    \bp_i
    &\coloneqq
    \begin{bmatrix}
        \mathbf{0}_{d^\prime}
        \\
        \be_i
        \\
        1
        \\
        1\{ i < {T} \}
    \end{bmatrix}
    \in \R^{d^\prime + T + 3}
    ,
\end{align}
where $\be_i$ is an one-hot vector with 1 at the $i$-th entry, and $d^\prime + T + 3 = D$.
Here, our goal is to predict $x_{T}$.
\begin{remark}\label{remark:format-assumption}
    Most in-context learning studies \cite{akyrek2023what, bai2024transformers, li2023transformers} make an assumption on the input data, where they assume it is formatted with features and labels in the same column, i.e.,
\begin{equation}\label{eqn:format-assumption}
\begin{bmatrix}
    \bx_1 &  \bx_2 & \dots & \bx_{N} 
    \\ 
    \yb_1 & \by_2 & \dots & \by_{N}
    \\
    \pb_1 & \pb_2 & \dots & \pb_{N}
\end{bmatrix}.
\end{equation}
In contrast, we adopt a natural approach that leverages the raw structure of the data, particularly for the $\mathtt{AR}_d(q)$ process. 
In this setting, each time step’s label also serves as a feature for future steps. 
Further, the unknown value of $q$ complicates the task of achieving such a format in Equation~\eqref{eqn:format-assumption}.
\end{remark}

Our next lemma shows that transformers are indeed capable of reformatting $\bfa H$ into the form of Equation~\ref{eqn:format-assumption}.
Notably, the following lemma relaxes the assumption in Remark~\ref{remark:format-assumption} of previous studies as well.

\begin{lemma}\label{lem:input-causal}
    Given a sequence of token $\bfa H$ in the form of Equation~\ref{eqn:input-data}, there exists a one-layer, $q_{\max}$ head attention layer, such that for any $q \leq q_{\max}$, the columns of $\text{Attn}_{\bm{\theta}}^{\dagger}( \bfa H )$ has the following form:
    \begin{equation}
    \text{Attn}_{\bm{\theta}_1}^{\dagger}( \bfa H )_i
    \coloneqq
        \begin{bmatrix}
            x_i
            \\
            x_{i-1}
            \\
            \vdots
            \\
            x_{i-q}
            \\
            \bp_i^\prime
        \end{bmatrix},
        \quad
        \bp_i^\prime 
        \coloneqq
        \begin{bmatrix}
        \mathbf{0}_{ d^\prime - q }
        \\
        \be_i
        \\
        1
        \\
        1 \{ i < T \}
        \end{bmatrix}.
    \end{equation}
\end{lemma}
The proof is in \cref{proof:lem-input-casual}.
\cref{lem:input-causal} is crucial in our analysis as it connects the theoretical results in ICL \cite{bai2024transformers} to uni-variate time series forecasting.
When data formats in the form of Equation~\eqref{eqn:format-assumption}, \cite{bai2024transformers} show that there exists a multi-layer transformer that performs linear regression via gradient descent on the first $N-1$ data points and evaluates the $N$-th one.
Thus, \cref{lem:input-causal} implies transformers are also capable of performing linear regression on time series data, which we present in the following paragraph.

This lemma applies to both any-variate attention and standard attention, as the latter can be viewed as a special case of any-variate attention by setting $u^1, u^2 = 0$.
Additionally, the construction of a single layer with $q$ heads is not a strict requirement; the lemma also holds for $c$ layers of  $\frac{q}{c}$ head attention, for any $c$ satisfies $\frac{q}{c} >= 2$.

With Lemma~\ref{lem:input-causal}, we are able to apply the in-context learning results in \cite{bai2024transformers} on the $\mathtt{AR}_1(q)$ case.
Consider the data generated by the $\mathtt{AR}$ process in Equation~\ref{eqn:AR-data}.
Given an input time series $\bx \in \R^{ d \times T}$, we define the least squares estimator as the empirical risk minimizer over the time series, i.e.,
{\small
\begin{align*}
    \ell_{\text{reg}}(\bw, \bx_{t-1:t-q}  )
    &\coloneqq
    \frac{
    \left[
    \langle
    \bw ,
    [ \bx^1_{t-1:t-q} ; \dots ; \bx^d_{t-1:t-q}]
    \rangle
    -
    x_t^1
    \right]^2
    }{2}
    \\
    L_{\text{reg}}
    (\bw, \bx)
    &\coloneqq
    \frac{1}{T-1}
    \sum_{t=1}^{T-1}
    \ell_{\text{reg}}
    \left(
    \bw, \bx_{t-1:t-q}
    \right)
    \\
    \hat{\bw}_{\text{ERM}} 
    &\coloneqq
    \argmin_{\bw\in\R^{dq}}
    \;
    L_{\text{reg}}
    \left(
    \bw, \bx
    \right),
\end{align*}
}
where $[ \bv ; \bu]$ denotes the concatenation between vectors, as $[\bx^1_{t-1:t-q} ; \bx^2_{t-1:t-q}] = ( \bx^1_{t-1},  \bx^1_{t-2},  \cdots, \bx^2_{t-q+1}, \bx^2_{t-q} ) \in \R^{2q}$, $\tilde{\bx}$ denotes masking out the last time step of the target variate, and $L_{reg}$ is a loss, which is $\alpha$-strongly convex, and $\beta$-smooth over $\R^{dq}$.
We make the following assumption and then present our first result on uni-variate time series ($d=1$).
\begin{assumption}\label{assumption:effective-regression}
    The regression problem above $\hat{\bw}_{\text{ERM}}$ is well-conditioned and has a bounded solution.
\end{assumption}
\begin{proposition}[Uni-variate Autoregressive Regression via Transformers]
    Assume \cref{assumption:effective-regression} holds and fix a $q_{\max} > 0$.
    For any $0 \leq \alpha \leq \beta$ with $\kappa \coloneqq \frac{\beta}{\alpha}$, $B_w > 0$, and $\epsilon < B_x B_w / 2$,
    there exists a $L$-layer transformer
     $\text{TF}_{\bm{\theta}}^{0 \dagger} 
    \left( \cdot \right)$, with
    \begin{align*}
        L = L_1 + L_2, \quad
        L_1& = \lceil 2 \kappa \log( \frac{B_x B_w}{2 \epsilon} ) \rceil,
        \quad
        L_2 = \lceil \frac{q_{\max}}{3} \rceil,
        \\
        \text{max}_{\ell \in [L]} 
        M^{(\ell)} 
        &\leq 3,
        \quad
        \norm{ \bm{\theta} }_{\text{op}}
        \leq
        | 4R + 8\beta^{-1} |
        ,
    \end{align*}
    ($R \coloneqq \text{max}\{ B_x B_w, B_x, 1 \}$), the following holds.
    On any input data $\bx$ generated by any $\mathtt{AR}_1(q)$ process such that
    \begin{equation}
        0 < q \leq q_{\max}
        \quad
       \norm{ \hat{\bw}_{\text{ERM}} }_2
       \leq 
       \frac{B_w}{2},
    \end{equation}
    we have
    \begin{equation}
        \lVert
        \hat{\bx}_{T}
        -
        \left\langle
        \hat{\bw}_{\text{ERM}}, 
        [
        \bx_{t-1:t-q}^1 ; \dots 
        ; \bx_{t-1:t-q}^d
        ]
        \right\rangle
        \rVert
        \leq 
        \epsilon,
    \end{equation}
    where 
    $\hat{\bx}_{T} = \mathtt{read}(\text{TF}_{\bm{\theta}}^{0 \dagger}
    \left( \bfa H \right))$.
    The $\mathtt{read}(\bfa H)$ operation reads out the first entry of $T$-th column of $\bfa H$.
\end{proposition}
This proposition follows immediately from \cref{lem:input-causal} and \citep[Theorem~4]{bai2024transformers}.
The above result applies for MOIRAI with $u^1_m, u^2_m = 0$ in all heads and layers.
Further, one can replace the least squares ERM with lasso or ridge ERM and obtain a similar result by applying Theorem 4, 7, and 13 of \cite{bai2024transformers}.

So far, we show that transformers are capable of solving uni-variate autoregressive regression with, at best, one additional layer compared to the results in \cite{bai2024transformers}.
The result above provides insights on transformer-based uni-variate time series foundation models \cite{ansari2024chronos, rasul2023lag, das2023decoder}.
To study MOIRAI, we then include two ingredients into our analysis: the \textit{any-variate encoding} and the \textit{covariates} in the following chapters.

\subsection{Approximation Error of MOIRAI Transformer}
In this subsection, we extend our results to the multivariate autoregressive process ($d > 1$) and the encoding method of MOIRAI.
Note that in the multi-variate case, we only focus on MOIRAI as it is the only transformer-based model that is compatible with arbitrary number of covariates. 
We start by introducing the any-variate encoding.
\paragraph{Any-Variate Encoding.}
\cite{woo2024unified} propose to flatten a $d$-dimensional time series, $\bx \in \R^{d \times T}$, into a $1$-dimensional sequence, i.e., $\bx^\prime \in \R^{1 \times Td}$.
This operation transforms time series with arbitrary number of covariates ($d$), into a long sequence with fixed dimension, enabling consistent input dimension for transformers.
Following the flattening operation, \cite{woo2024unified} also proposes to add two types of indices into the input sequence: the time and variate ID.
We term the above operations as the any-variate encoding, which transforms a multivariate sequence $\bx \in \R^{d \times T}$, as follows:
{    \small
\begin{equation}\label{eqn:AV-encoding}
    \begin{bmatrix}
        x_1^1 & \cdots & x_T^1 
        \\
         x_1^2 & \cdots & x_T^2
         \\
         \vdots & \vdots & \vdots 
         \\
        x_1^d & \cdots & x_T^d    
    \end{bmatrix}
    \rightarrow
    \begin{bmatrix}
        x_1^1 & \cdots & \cellcolor{cyan!20} x_{T}^1 &  \cdots & x_1^d & \cdots & x_{T}^d
        \\
        \bp_1 & \cdots & \bp_{T} &  \cdots & \bp_1 & \cdots & \bp_T
        \\
        \be_1 & \cdots & \be_1 &  \cdots & \be_d & \cdots & \be_d
    \end{bmatrix},
\end{equation}}
where $\be_i$ is the variate index, a one-hot vector with $i$-th entry being $1$, and $\bp_i$ is the time index, which is defined the same as Equation~\eqref{eqn:input-data}.
This is without loss of generality because the discrete-time and variate ID used in \cite{woo2024unified} can be easily transformed into a high-dimensional vector with the embedding layer.
Note that only the target variate has length $T$, we highlight $x_{T}^1$ as it is our prediction target and will be masked as $0$.

Now we define the history matrix $\mathtt{A}_i(q) \in \R^{q+1 \times T}$ for the $i$-th covariates $(x_1^i, \cdots, x_T^i)$, with order $q$, such that
\begin{equation*}
    \mathtt{A}_i(q)_{\mu, \nu}
    \coloneqq
        x^i_{\nu - \mu + 1},
        \quad
        \text{ for }\mu\in[d], \nu\in[q].
\end{equation*}
where in the $j$-th column of $\mathtt{A}_i(q)$, it contains historical values of $x_j^i$ with lag $q > 0$.

\begin{lemma}\label{lem:mar-group-wise}
Fix $q_{\max}, D \in \N^+$.
Given any $T > 0, d^\prime > q > 0, d > 0$ such that $T > q$, $q_{\max} \geq q$.
For any input matrix $\bH$ in the form of any-variate encoding in Equation~\ref{eqn:AV-encoding}, such that $\bH \in \R^{ D \times dT}$.
There exists a one layer, $q_{\max}$ head \textbf{any-variate attention} that performs the following operation.
    \begin{align*}
    &
    \begin{bmatrix}
        x_1^1 & \cdots & x_{T}^1 & x_1^2 & \cdots & x_T^2 & \cdots & x_1^d & \cdots & x_T^d
        \\
        \bp_1 & \cdots & \bp_{T} & \bp_1 & \cdots & \bp_T & \cdots & \bp_1 & \cdots & \bp_T
        \\
        \be_1 & \cdots & \be_1 & \be_2 & \cdots & \be_2 & \cdots & \be_d & \cdots & \be_d
    \end{bmatrix}
    \\ 
    &\quad\quad\quad\quad\mapsto
    \begin{bmatrix}
        \mathtt{A}_1(q) & \mathtt{A}_2(q) & \cdots & \mathtt{A}_d(q)
        \\
        \bm{0}_{d^{''} \times T} & \bm{0}_{d^{''} \times T} & \cdots & \bm{0}_{d^{''} \times T}   
        \\
        \ddots & \ddots & \cdots & \ddots 
    \end{bmatrix},
    \end{align*}
    where $d^{''}  = d^\prime - q_{\max}$.
\end{lemma}
The proof is in \cref{proof:any-var-enc}.
Intuitively, the above operation performs the same operation in Lemma~\ref{lem:input-causal} but in a variate-wise fashion.
\cref{lem:mar-group-wise} shows that any-variate attention is capable of organizing the history of each variate efficiently.
To again achieve the format in Equation~\eqref{eqn:format-assumption}, one has to stack all $\mathtt{A}_i(q)$ in the same columns, which can be easily done by a single layer of attention via \cref{lem:input-causal} and \citep[Proposition~A.5]{bai2024transformers} (details in \cref{proof:any-var-enc}).
This lemma serves as a foundation for MOIRAI to handle multi-variate time series with in-context learning which we present as the theorem below.
\begin{remark}
        Comparing to Lemma~\ref{lem:input-causal}, \cref{lem:mar-group-wise} is specifically for any-variate attention in our construction, where we demonstrate that several special mechanisms in any-variate attention enables variate-wise operations efficiently.
\end{remark}
\begin{remark}    
        Lemma~\ref{lem:input-causal} and Lemma~\ref{lem:mar-group-wise} can be generalized to Softmax and linear attention by considering perturbations, making them applicable to a wide range of transformers.    
\end{remark}

\begin{theorem}[Any-variate Autoregressive Regression via MOIRAI]\label{thm:any-variate-auto}
    Assume \cref{assumption:effective-regression} holds.
    For any $0 \leq \alpha \leq \beta$ with $\kappa \coloneqq \frac{\beta}{\alpha}$, $B_w > 0$, and $\epsilon < B_x B_w / 2$.
    there exists an $(L_1+L_2)$-layer of MOIRAI transformer equipped with any-variate Attention, satisfies the following
    \begin{align*}
        &
        L_1 = \lceil  \frac{q_{\max}}{3} \rceil  +1, \quad
        L_2 = \lceil 2 \kappa \log \frac{ B_x B_w}{2\epsilon} \rceil, 
        \quad
        \\
        &\quad\quad\quad \quad\quad\max_{\ell \in [L_1+1, L_2]} M^{(\ell)} \leq 3, 
        \\
        &\vertiii{\bm{\theta}} \leq \lvert 4R + 8 \beta^{-1} \rvert,
        \quad \sum_{ \ell=1 }^{L_1} M^{(\ell)} = d_{\max} + q_{\max},
    \end{align*}
    where $d_{\max} > 0$.
    For any input time series $\bx$ with length $T$ generated from an $\mathtt{AR}_d(q)$ process, where
    \begin{align*}
        \bx \in \R^{d \times T},\quad q \leq q_{\max}, \quad d \leq d_{\max}.
    \end{align*}
    Then there exists a MOIRAI transformer with $D \geq (q+1)d_{\max} + T + 2$, satisfies the following
    \begin{equation}
        \lVert
        \hat{\bx}_{T}^1
        -
        \left\langle
        \bw^\star_i, [\bx_{T-1: T-q}^1; \dots ; \bx_{T-1: T-q}^d]
        \right\rangle
        \rVert
        \leq 
        \epsilon,
    \end{equation}
    where $\hat{\bx}_T^1 = \mathtt{read}(\text{TF}^0_{\bm{\theta}}(\bH))$, and $\bH \in \R^{D\times  N}$ is the any-variate encoding of $\bx$.
\end{theorem}
\begin{remark}
    \cref{thm:any-variate-auto} indicates there exists a MOIRAI transformer that fits an autoregressive model on time series as long as the number of covariates no greater than $d_{\max}$ and lags no greater than $q_{\max}$.
    This shows its ability to infer the underlying AR model in a principled way and provides a possible explanation for its zero-shot performance on a wide range of datasets.
\end{remark}

The proof is in \cref{proof:group-wise}.
Observe that there exists two trade-offs in \cref{thm:any-variate-auto}.
First, $q_{\max} d_{\max}$ is upper bounded by the hyperparameter $D$ (up to constant), which is a natural trade-off in our construction.
Second, the approximation error is roughly $O( e^{-L} )$, suppressed exponentially by the number of layers, as in our analysis, each layer of MOIRAI performs a single step of gradient descent on $L_{\text{reg}}$.

Another popular approach of time series prediction is through probabilistic forecasting, where the model estimates the distribution from input data.
In \cref{thm:moirai-mle}, we show that there also exists a MOIRAI that performs Maximum Likelihood Estimation with a small estimation error.







\section{Generalization}
In this section, we investigate the generalization bound of pretraining transformer-based time series foundation models.
This section will focus on learning MOIRAI on multi-variate time series, one can easily adapt our proofs into learning uni-variate time series with standard transformers.

Let $\pi$ be a meta distribution, and each distribution drawn from it $ \mathtt{P}^{(T)} \sim \pi$, satisfies Dobrushin's condition \cite{Dobrushin1968TheDO} (which we will introduce shortly).
For pretraining data, we first sample $n$ distributions $\mathtt{P}_j^{(T)}$ i.i.d. from $\pi$, and for each distribution, we sample a time series $(\bx_{1j}, \cdots, \bx_{Tj})$, for $j \in[n]$, and each of them contains no more than $d$ covariates and with lag step no more than $q$.

For each time series, we encode it with any-variate encoding into an input matrix denoted as $\bH \in \R^{D \times N}$, \footnote{Due to any-variate encoding, $N = dT$.}
We define each pretraining sample as $\bz_j \coloneqq \left( \bH_j, y_j \right)$, where $y_j = \bx_{Tj}^1$.
We consider the squared loss between model prediction and the label, i.e.
{\small
\begin{equation*}
    \ell( \bz_t, \bm{\theta} )
    \coloneqq
    \frac{1}{2}
    \Bigg[
    y_t
    -
    \mathtt{Clip}_{B_x}
    \bigg(
    \mathtt{read}_y
    \Big(
    \text{TF}_{\bm{\theta}}^R
    \left(
    \bH
    \right)
    \Big)
    \bigg)
    \Bigg]^2,
\end{equation*}
}
where 
$\mathtt{Clip}_{B_x}(t) \coloneqq \max\{ \min \{ t, B_x \}, -B_x \}$, and $\text{TF}_{\bm{\theta}}^R$ is the MOIRAI transformer defined in \cref{def:moirai} with $\mathtt{Clip}(\cdot)$ applied after each layer.
The pretraining loss and test loss is defined as the following:
\begin{equation}\label{eqn:icl-loss}
    \hat{L}( \bm{\theta} )
    \coloneqq
    \frac{1}{nT}
    \sum_{t=1}^T
    \sum_{j=1}^n
    \ell( \bm{\theta}, \bz_{jt} ),
    \;
    L( \bm{\theta} )
    \coloneqq
    \mathbb{E}_{\bz,\mathtt{P}^{(T)}}
    \left[
    \ell(  \bm{\theta}, \bz )
    \right].
\end{equation}
The goal of our pretraining algorithm is to find an empirical risk minimizer (ERM) over MOIRAI transformers with $L$ layers, $M$ heads, and norm bounded by $B$:
\begin{align}\label{eqn:parameter-regime}
    &\hat{\bm{\theta}}
    \coloneqq
    \underset{ \bm{\theta} \in \Theta_{L,M,D^\prime, B} }{\argmin}
    \hat{L}
    (\bm{\theta}),
    \\
    &\Theta_{L,M,D^\prime,B}
    \coloneqq
    \Bigg\{
    \bm{\theta}
    =
    \left(
    \bm{\theta}_1^{(1:L)},
    \bm{\theta}_2^{(1:L)}
    \right)
    :
    \\
    \max_{\ell\in[L]}
    &M^{(\ell)} \leq M
    , 
    \quad
    \max_{\ell\in[L]}
    D^{(\ell)} \leq D^\prime
    ,
    \quad
    \Vert\bfa\theta\Vert_{op}
    \leq 
    B
    \Bigg\}.
\end{align}
\subsection{Weakly-Dependent Time Series}
In this scenario, we consider the training data $\bx$ to be drawn from a distribution $\mathtt{P}$ satisfying Dobrushin's condition.
Under this condition, we are able to present several generalization bounds on pretraining.

\begin{definition}[Influence in high dimensional distributions]\label{def:influence}
    Let $\cX = (\cX_1, \cdots , \cX_T)$ be a sequence of random variables over $\cD_{\cX}^{T}$.
    The influence of variable $\cX_j$ on variable $\cX_i$ is defined as
    \begin{align*}
        &\bI_{j \rightarrow i}(\cX)
        \coloneqq
        \max_{ \text{x}_{ -i-j}, \text{x}_j, \text{x}_j^\prime  }
        \\
        &
        \norm{
        P_{\cX_i | \cX_{-i}} \left( \cdot | \text{x}_{-i-j}, \text{x}_j \right),
        P_{\cX_i | \cX_{-i}} \left( \cdot | \text{x}_{-i-j}, \text{x}_j^\prime \right)
        }_{\texttt{TV}}
        ,
    \end{align*}
    where $\text{x}_{-i-j} \in \cD_{\cX}^{T-2}, \text{x}_j, \text{x}_j^\prime \in \cD_{\cX}$, $\norm{\cdot}_{\texttt{TV}}$ denotes the total variation distance, and $\text{x}_{-i}$ represents the vector \textbf{x} after omitting the $i$-th element.
\end{definition}

\begin{definition}[Dobrushin's Uniqueness Condition]
    Consider a random variable $\cX$ over $\cD_{\cX}^T$.
    The Dobrushin coefficient of $\cX$ is defined as 
    \[
    \alpha(\cX) \coloneqq  \max_{1\leq i \leq T} \sum_{j \neq i} \bI_{j \rightarrow i} (\cX).
    \]
    We say the variable satisfies Dobrushin's uniqueness condition if $\alpha(\cX) < 1$.
    For a distribution $\mathtt{P}$, we denote $\alpha(\mathtt{P}) = \sup_{\cX\sim\mathtt{P}} \alpha(\cX)$.
\end{definition}
\begin{definition}[Log Dobrushin's Coefficients]
        Let $\cX = (\cX_1, \cdots, \cX_T)$ be a random variable over $\cD_{\cX}^T$ and let $\mathtt{P}_z$ denote its density.
        Assume that $\mathtt{P}_{z} > 0$ on all $\Omega^T$.
        For any $i \neq j \in [T]$, the log influence between $j$ and $i$ is defined as:
    {\small
    \begin{equation*}
        I^{\log}_{j, i}(\cX)
        =
        \frac{1}{4}
        \sup
        \log
        \frac{ P
        \left[
        \text{x}_i, \text{x}_j, \text{x}_{-i-j}
        \right] 
        P
        \left[
        \text{x}_i^\prime, \text{x}_j^\prime, \text{x}_{-i-j}
        \right]
        }{
    P
        \left[
        \text{x}_i^\prime, \text{x}_j, \text{x}_{-i-j}
        \right] 
    P
        \left[
        \text{x}_i, \text{x}_j^\prime, \text{x}_{-i-j}
        \right] 
        },
    \end{equation*}
    }
    where the $\sup$ is taken over $\text{x}_{-i-j}, \text{x}_i, \text{x}_i^\prime, \text{x}_j, \text{x}_j^\prime$,
    and the log-coefficient of $\cX$ is defined as $\alpha_{\log}(\cX) = \max_{i \in [T]} \sum_{j \neq i} I^{\log}_{j, i}(\cX)$.
\end{definition}
The coefficient $\alpha(\cdot)$ has a natural bound $0 \leq \alpha(\cdot) \leq T-1$, with $\alpha = 0$, the data reduces to the i.i.d. case.

\begin{remark}
    Dobrushin's condition characterizes a class of distributions whose dependency is mild. 
    However, our empirical evaluation suggests that in certain situations where Dobrushin's condition fails to hold, the Transformers can perform prediction well.
\end{remark}
\subsection{Generalization Bounds of MOIRAI}

\begin{theorem}[Pretraining Generalization Bound]\label{thm:gen-bound-1}
    Let $\Theta_{L,M,D^\prime, B}$ be the parameter space defined in 
    Equation~\ref{eqn:parameter-regime}.
    Assume $\alpha_{\log}( \mathtt{P}^{(T)}) < 1/2$.
    Then with probability at least $1 - \varepsilon$, ERM $\hat{\bm{\theta}}$ satisfies the following:
    \begin{align*}
    L(\hat{\bm{\theta}})
    &\leq 
    \inf_{\bm{\theta} \in \Theta_{L,M,D^\prime, B}} L(\bm{\theta})
    +
    \\
    &
    O
    \left(
    \frac{B_x^2}{1 - \alpha(\mathtt{P}^{(T)}) }
    \sqrt{
    \frac{
    L(MD^2 + D D^\prime) \zeta + \log(\nicefrac{1}{\varepsilon})
    }{n}
    }
    \right),
    \end{align*}
    where $C$ is an universal constant, and $\zeta = O(\log(2 + \max \{ B, \mathtt{R}, B_x, T, d \}$. 
\end{theorem}

The proof is in \cref{proof:gen-bound-1}.
Note that when $\alpha(\mathtt{P}) = 0$, the data becomes i.i.d., where the only difference between our generalization and one proposed in \cite{bai2024transformers} is the complexity term. 
The complexity of MOIRAI and standard transformers differs as the complexity of MOIRAI also dependents on the time series length ($T$).
Further, in \cref{thm:gen-bound-1}, we do not assume our data is generated from the $\mathtt{AR}$ process, only its Dobrushin coefficient.
When the data is generated by the $\mathtt{AR}$ process, we are able to give a more explicit bound on the same test loss as described below.

\begin{corollary}[Test Error Bound]\label{thm:test-error-bound-1}
    Following the setup in \cref{thm:gen-bound-1},
    if pretraining samples are generated by some $\mathtt{AR}_d(q)$ process with noise sampled from $N(0, \sigma^2_\epsilon)$\footnote{Here we assume fixed $d, q$ across all samples as one can describe a lower dimension/order $\mathtt{AR}$ process with zero coefficients.}, 
    then with probability $ \Delta(1 - \varepsilon)$, ERM $\hat{\bm{\theta}}$ satisfies the following:
    \begin{align*}
    L(\hat{\bm{\theta}})
    &
    \leq
    O
    \Bigg(
    B_x B_w \exp \left( \frac{-L}{\kappa} \right)
    +
    \\
    &
    \frac{B_x^2}{1 - \alpha(\mathtt{P}^{(T)})}
    \sqrt{ \frac{L(MD^2 + D D^\prime ) \zeta + \log (1 / \varepsilon)}{n} }
    \Bigg).
    \end{align*}
    where $\Delta = O\left(1 - \left(  \nicefrac{\sigma_\epsilon}{B_x  B_w e^{\nicefrac{-L}{2\kappa}}}  \right)^2  \right)$, $C$ is an universal constant, and $\zeta = O(\log(2 + \max \{ B, \mathtt{R}, B_x, T, d \})$. 
\end{corollary}
\begin{remark}
Considering the model parameters ($M,D,D^\prime, d$) are of constant level, one is able to further optimize the bound to
$L(\hat{\bm{\theta}}) \lesssim  n^{-\nicefrac{1}{2}}$, by selecting $L$ appropriately.    
\end{remark}


\subsection{Example: Stationary $\mathtt{AR}(1)$}\label{sec:AR1}
Here we provide an example of the application of \cref{thm:test-error-bound-1} on $\mathtt{AR}(1)$ process with the following form
\begin{equation*}
    \bx_{t+1}
    =
    \langle \bw, \bx_{t} \rangle + \epsilon_t, \quad \epsilon \sim N(0, \sigma_{\epsilon}^2),
\end{equation*}
where $\bx_t \in \R^d$, $\bw \in \R^d, \epsilon \in \R$ and $\by_{t+1} = \bx_{t+1}^1$.

To satisfy the condition of $\alpha( \mathtt{P} ) < \frac{1}{2}$, we assume the following holds
\begin{equation}\label{eqn:condition-weakly-dependent-AR1}
    B_x^2 < \ln \frac{1}{2} + ( \sigma_{\epsilon}^2 ),
    \quad  \norm{\bw}_{\infty} < 1.
\end{equation}
The first condition comes from the fact that we require the pair-wise potential of this time series to be less than $1/2$ (For more details, see \cref{appendix:analysis-ar1}).
The second condition comes from the requirement of it being stationary.
\begin{proposition}[Generalization Bound for Any-Variate Transformer on $\mathtt{AR}(1)$]\label{proposition:ar1}
    Considering an $\mathtt{AR}(1)$ process with Dobrushin's coefficient bounded by $1/2$.
    With probability at least $\delta(1 - \varepsilon)$, ERM $\hat{\bm{\theta}}$ satisfies the following:
    \begin{align*}
    L(\hat{\bm{\theta}})
    &=
    O
    \Bigg(
    \frac{\sigma_\epsilon}{\sqrt{1 - \delta}}
    +
    \frac{\sigma_\epsilon^2}{B_x}
    \exp \left( \frac{-L}{\kappa} \right)
    +
    \\
    &
    \frac{\sigma_\epsilon^2}{1 - \alpha( \mathtt{AR}(1) )}
    \sqrt{ \frac{L(MD^2 + D D^\prime) \zeta + \log (1 / \varepsilon)}{n} }
    \Bigg).
    \end{align*}
    where $\zeta = O(\log(2 + \max \{ B, \mathtt{R}, B_x, d \})$. 
\end{proposition}
If we further optimize the bound by viewing the hyperparameters as constants, the test error obeys $O(e^{-L} + \sqrt{\frac{L}{n}})$ with high probability whenever $\sigma_\epsilon$ is small.

\vspace{-1em}
\section{Experiments}
\begin{figure*}[t]
    \centering
    \includegraphics[width=\linewidth]{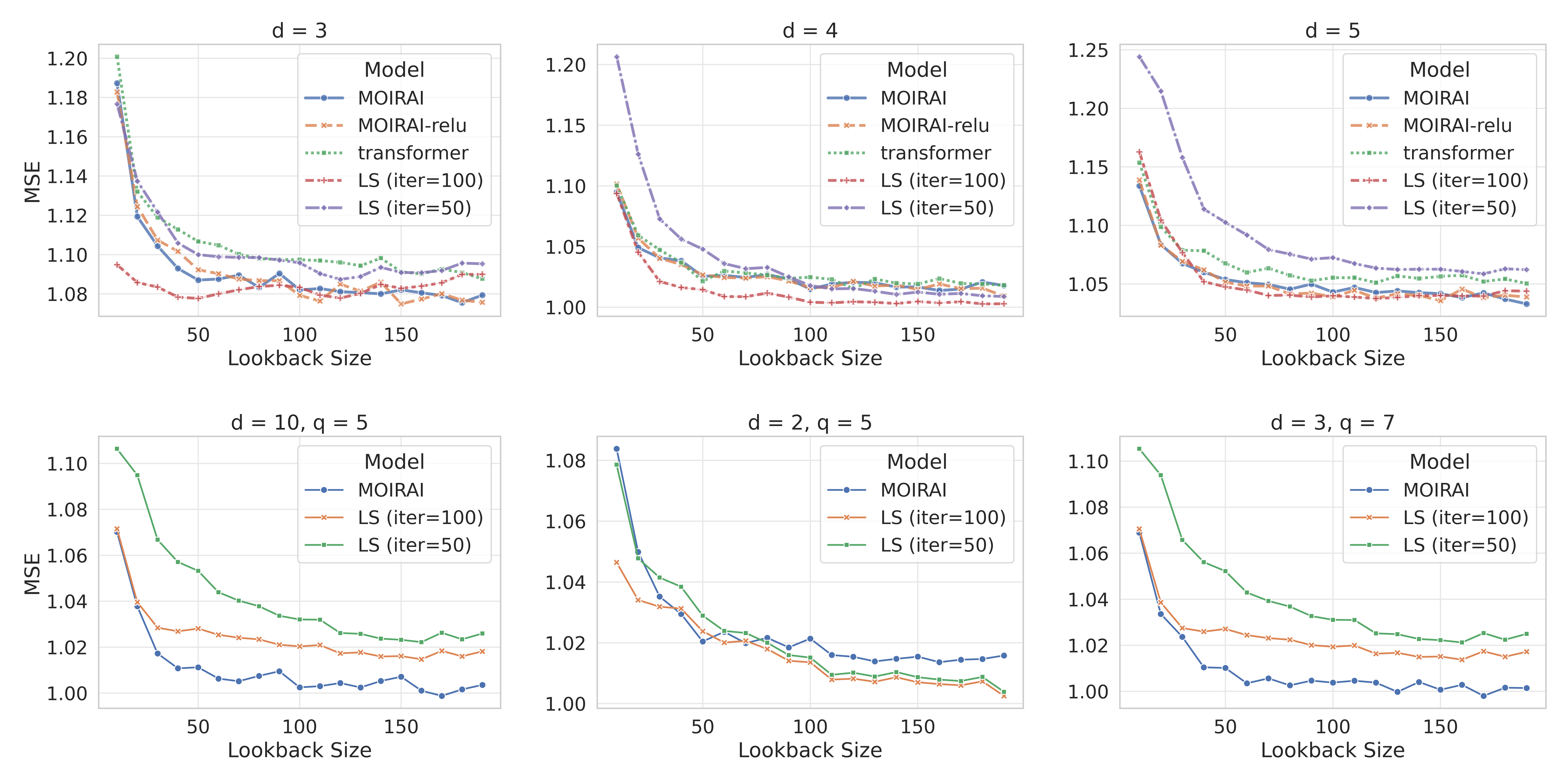}
    \caption{
    \textbf{Top: Model performance on data with different number of covariates.
    }
    For both MOIRAI and MOIRAI-relu, we observe their performance behave like least squares.
    As in our construction, the longer the lookback size is, the more examples available for transformers to fit an $\mathtt{AR}$ model.
    Note that our test data has variance $\sigma^2 = 1$, thus the MSE for both models are expected to converge to $1$ as the lookback size increases.
    \textbf{Bottom: Generalization to unseen values of $d, q$.}
    From left to right, we have MOIRAI's generalization performance (pretrained on $d\in\{4,5\}, q\in\{4,5\}$) on high dimensional data ($d=10$), low dimensional data ($d=2$) and high lag step + low dimensional data ($d=3,q=7$).
    Note that high and low is compared with pretraining data.
    We observe that even when MOIRAI did not learn from any time series with $d=10$, it is still able to generalize well and shows even better sample complexity than least squares regression.
    Finally, even when both $q,d$ are unseen, it does not impact MOIRAI's ability to make accuracy predictions.
    }
    \label{fig:icl-results}
\end{figure*}

To verify our analysis, we first train transformers on synthetic datasets generated from $\mathtt{AR}$ process with different parameters.
The goal of this experiment is to verify the existence of a transformer that performs least squares regression on input time series with bounded lag window and number of covariates.
Next we study whether a pretrained transformer is capable of generalize such an ability to unseen values of $d, q$.
More empirical results are in \cref{appendix:additional-exp}.
\subsection{Datasets}
\paragraph{Synthetic Data Generation.}
We generate the $\mathtt{AR}$ synthetic data similar to Equation~\eqref{eqn:AR-data} but use normalization to stabilize the values.
Consider a sequence of data 
$\bx \in \R^{d \times T} \coloneqq (\bx_1, \dots, \bx_T)$, where $\bx_t = (x_t^1, \cdots, x_t^d) \in \R^d$.
Assuming our target (variate of interest) is in dimension $1$, we generate our data as follows:
\begin{equation}
    x_{t}^1
    =
    \frac{1}{qd}
    \sum_{i=1}^q
    \sum_{j=1}^d
    a_i^j \cdot x_{t-i}^j
    + \epsilon_t
    ,
\end{equation}
where $\epsilon_t \sim N(0, \sigma^2)$, $a_i^j \sim N(0, 1) \in \R$.
We have $\sigma^2 \sim \text{unif}(0.1, 1)$.
After recursively generating the time series, we remove its first 50 time steps as burn out.
Each $\mathtt{AR}$ time series has the number of covariates between $1$ to $5$ and lag between $1$ to $5$.
For test data, we randomly generate one time series with $5k$ data points with $\sigma^2 = 1$, and evaluate our model on all time steps.
We set $q, d \leq 5$ in our experiments.
In total, we generate $100$ different time series with randomly sampled $d$ and $q$.
We also conduct experiments on synthetic data with seasonality, which can be found in the appendix.
\paragraph{Model.}
We use MOIRAI-base, it is a $12$ layer MOIRAI transformer, with hidden dimension $768$.
The hyperparameters of this experiment can be found in \cref{table:hyperparameters}.
We use AdamW optimizer with linear warm ups.
We use MSE loss for pretraining, comparing to \cite{woo2024unified} using NLL loss, we choose MSE loss to simplify our settings.
\paragraph{Training and Evaluation.}
For pretraining, we follow the standard MOIRAI pretraining but set the patch size as $1$ to minimize the impact of patch embedding.
During pretraining, each time series is randomly sampled, and the mask is randomly applied to each time step with probability $0.15$.
We evaluate the pretrained model on our test data with $d = \{3, 4, 5\}$, $q=5$ and $\sigma^2 = 1$ with different input length.
\paragraph{Baselines.}
We compare MOIRAI with least squares regression performing different gradient descent steps.
For least squares regression, we assume $q$ is known.
When MOIRAI takes a $T$ length input, the least squares regression is trained on $T-1$ samples with each having $dq$ features.
A more detailed example on how we implement baselines is in \cref{appendix:ls-baseline}.
We also include the standard transformers and MOIRAI with ReLU replacing Softmax, which we term it as MOIRAI-relu.
For standard transformers, we keep the any-variate encoding but replace its attention with standard attention.
In \cite{woo2024unified}, without any-variate attention, the error of MOIRAI-small increases roughly $40\%$.
\paragraph{Results.}
Since our test data generation process obeys noise variance = $1$, when fitting a linear model, the expected MSE will converge towards $1$ as lookback size increases.
Based on \cref{thm:any-variate-auto}, the length of input time series also corresponds to the number of examples model perform least squares on via gradient descent.
We observe that as the input length increases, the predictive error of MOIRAI decreases similarly to least squares, which verifies \cref{thm:any-variate-auto}.
Next, when pretrained on diverse dataset, pretrained MOIRAI is able to adapt to different number of covariates and perform least squares accordingly.
Further, when replace softmax with ReLU, the performance gap is negligible.
For standard transformer, while it also behaves similar to other models, it does present higher error comparing to other baselines, indicating the advantages of using any-variate attention.
\subsection{Generalization to Unseen $d, q$}
Here we are interested in whether a pretrained transformer is capable of generalizing to unseen values of $d$ and $q$.
Therefore, we train transformers (MOIRAI) on synthetic data generated with $\mathtt{AR}$ with $d \in \{4, 5\}$, and $q \in \{4, 5\}$.
In our construction, pretrained transformer is compatible with lower order and dimension $\mathtt{AR}$ data.
We evaluate the trained model on data with unseen values of $d$.
We select $d=2, d=10$, to represent the scenario when the number of covariates is lower and higher than pretraining data.
\paragraph{Results.}
We observe that even when facing data with unseen number of covariates, MOIRAI is still capable of performing least squares regression effectively.
Note that for $d=10$, least squares require higher sample complexity to obtain similar performance to $d=5$ cases.
However, the pretrained MOIRAI is able to outperform it from such an aspect.
For $d=2$ all models perform well, again verifies our theoretical results.
Finally, when facing data with unseen both $d$ and $q$, it is still capable of performing well.
\section{Conclusion}
In this paper, we investigate the theoretical foundations of transformers as time series foundation models.
First, we show that there exists a multi-layer transformer capable of performing least squares regression on any input uni-variate time series.
Next, when considering MOIRAI, we demonstrate the existence of a multi-layer MOIRAI that adapts to the dimensionality $d$ of the input (i.e., the number of covariates) and fits different autoregressive models based on $d$.
When the data is generated by an autoregressive process, such a transformer benefits from its prediction error being exponentially suppressed by the number of layers.
We then establish a generalization bound for pretraining when the data satisfies Dobrushin's condition.
When the pretraining data is sampled from $\mathtt{AR}$ processes, we derive a more explicit bound on the test loss, with a trade-off controlled by the number of layers.
Our analysis not only provides the first theoretical justification for the design and performance of MOIRAI but also represents the first theoretical framework for constructing a time series foundation model.

\paragraph{Limitations.}
One limitation in our analysis is that we consider ReLU instead of softmax in attention mechanisms.
While the same approach also is in theoretical \cite{bai2024transformers, lin2023transformers, he2025learning} and empirical works \cite{wortsman2023replacing, zhang2021sparse, shen2023study}, one might obtain a different approximation bound comparing to \cref{thm:any-variate-auto}.
However, in our generalization analysis, the difference is small as softmax does not affect the model complexity too much.
Another aspect is that we mainly focus on $\mathtt{AR}$ processes.
While in the appendix, we do show the approximation result for non-linear $\mathtt{AR}$ processes generated by a ReLU network, to achieve universal forecasting, a more general assumption on data is required.

\paragraph{Impact Statement.}
This paper studies the theoretical aspect of transformers as time series foundation models.
No negative societal impacts that the authors feel should be specifically highlighted here.

\bibliography{ref}

\begin{thebibliography}{10}

\bibitem{woo2024unified}
Gerald Woo, Chenghao Liu, Akshat Kumar, Caiming Xiong, Silvio Savarese, and Doyen Sahoo.
\newblock Unified training of universal time series forecasting transformers.
\newblock {\em arXiv preprint arXiv:2402.02592}, 2024.

\bibitem{ansari2024chronos}
Abdul~Fatir Ansari, Lorenzo Stella, Caner Turkmen, Xiyuan Zhang, Pedro Mercado, Huibin Shen, Oleksandr Shchur, Syama~Sundar Rangapuram, Sebastian~Pineda Arango, Shubham Kapoor, et~al.
\newblock Chronos: Learning the language of time series.
\newblock {\em arXiv preprint arXiv:2403.07815}, 2024.

\bibitem{liang2024foundation}
Yuxuan Liang, Haomin Wen, Yuqi Nie, Yushan Jiang, Ming Jin, Dongjin Song, Shirui Pan, and Qingsong Wen.
\newblock Foundation models for time series analysis: A tutorial and survey.
\newblock In {\em Proceedings of the 30th ACM SIGKDD conference on knowledge discovery and data mining}, pages 6555--6565, 2024.

\bibitem{das2023decoder}
Abhimanyu Das, Weihao Kong, Rajat Sen, and Yichen Zhou.
\newblock A decoder-only foundation model for time-series forecasting.
\newblock {\em arXiv preprint arXiv:2310.10688}, 2023.

\bibitem{rasul2023lag}
Kashif Rasul, Arjun Ashok, Andrew~Robert Williams, Arian Khorasani, George Adamopoulos, Rishika Bhagwatkar, Marin Bilo{\v{s}}, Hena Ghonia, Nadhir Hassen, Anderson Schneider, et~al.
\newblock Lag-llama: Towards foundation models for time series forecasting.
\newblock In {\em R0-FoMo: Robustness of Few-shot and Zero-shot Learning in Large Foundation Models}, 2023.

\bibitem{hamilton2020time}
James~D Hamilton.
\newblock {\em Time series analysis}.
\newblock Princeton university press, 2020.

\bibitem{mills1990time}
Terence~C Mills.
\newblock {\em Time series techniques for economists}.
\newblock Cambridge University Press, 1990.

\bibitem{Dobrushin1968TheDO}
P.~L. Dobrushin.
\newblock The description of a random field by means of conditional probabilities and conditions of its regularity.
\newblock {\em Theory of Probability and Its Applications}, 13:197--224, 1968.

\bibitem{dobrushin1987completely}
Roland~L Dobrushin and Senya~B Shlosman.
\newblock Completely analytical interactions: constructive description.
\newblock {\em Journal of Statistical Physics}, 46:983--1014, 1987.

\bibitem{bai2024transformers}
Yu~Bai, Fan Chen, Huan Wang, Caiming Xiong, and Song Mei.
\newblock Transformers as statisticians: Provable in-context learning with in-context algorithm selection.
\newblock {\em Advances in neural information processing systems}, 36, 2024.

\bibitem{von2023transformers}
Johannes Von~Oswald, Eyvind Niklasson, Ettore Randazzo, Jo{\~a}o Sacramento, Alexander Mordvintsev, Andrey Zhmoginov, and Max Vladymyrov.
\newblock Transformers learn in-context by gradient descent.
\newblock In {\em International Conference on Machine Learning}, pages 35151--35174. PMLR, 2023.

\bibitem{li2023transformers}
Yingcong Li, Muhammed~Emrullah Ildiz, Dimitris Papailiopoulos, and Samet Oymak.
\newblock Transformers as algorithms: Generalization and stability in in-context learning.
\newblock In {\em International Conference on Machine Learning}, pages 19565--19594. PMLR, 2023.

\bibitem{mahankali2023one}
Arvind Mahankali, Tatsunori~B Hashimoto, and Tengyu Ma.
\newblock One step of gradient descent is provably the optimal in-context learner with one layer of linear self-attention.
\newblock {\em arXiv preprint arXiv:2307.03576}, 2023.

\bibitem{ahn2024transformers}
Kwangjun Ahn, Xiang Cheng, Hadi Daneshmand, and Suvrit Sra.
\newblock Transformers learn to implement preconditioned gradient descent for in-context learning.
\newblock {\em Advances in Neural Information Processing Systems}, 36, 2024.

\bibitem{zhang2024trained}
Ruiqi Zhang, Spencer Frei, and Peter~L Bartlett.
\newblock Trained transformers learn linear models in-context.
\newblock {\em Journal of Machine Learning Research}, 25(49):1--55, 2024.

\bibitem{brown2020language}
Tom Brown, Benjamin Mann, Nick Ryder, Melanie Subbiah, Jared~D Kaplan, Prafulla Dhariwal, Arvind Neelakantan, Pranav Shyam, Girish Sastry, Amanda Askell, et~al.
\newblock Language models are few-shot learners.
\newblock {\em Advances in neural information processing systems}, 33:1877--1901, 2020.

\bibitem{garg2022can}
Shivam Garg, Dimitris Tsipras, Percy~S Liang, and Gregory Valiant.
\newblock What can transformers learn in-context? a case study of simple function classes.
\newblock {\em Advances in Neural Information Processing Systems}, 35:30583--30598, 2022.

\bibitem{su2024roformer}
Jianlin Su, Murtadha Ahmed, Yu~Lu, Shengfeng Pan, Wen Bo, and Yunfeng Liu.
\newblock Roformer: Enhanced transformer with rotary position embedding.
\newblock {\em Neurocomputing}, 568:127063, 2024.

\bibitem{akyrek2023what}
Ekin Aky{\"u}rek, Dale Schuurmans, Jacob Andreas, Tengyu Ma, and Denny Zhou.
\newblock What learning algorithm is in-context learning? investigations with linear models.
\newblock In {\em The Eleventh International Conference on Learning Representations}, 2023.

\bibitem{lin2023transformers}
Licong Lin, Yu~Bai, and Song Mei.
\newblock Transformers as decision makers: Provable in-context reinforcement learning via supervised pretraining.
\newblock {\em arXiv preprint arXiv:2310.08566}, 2023.

\bibitem{he2025learning}
Yihan He, Yuan Cao, Hong-Yu Chen, Dennis Wu, Jianqing Fan, and Han Liu.
\newblock Learning spectral methods by transformers.
\newblock {\em arXiv preprint arXiv:2501.01312}, 2025.

\bibitem{wortsman2023replacing}
Mitchell Wortsman, Jaehoon Lee, Justin Gilmer, and Simon Kornblith.
\newblock Replacing softmax with relu in vision transformers.
\newblock {\em arXiv preprint arXiv:2309.08586}, 2023.

\bibitem{zhang2021sparse}
Biao Zhang, Ivan Titov, and Rico Sennrich.
\newblock Sparse attention with linear units.
\newblock {\em arXiv preprint arXiv:2104.07012}, 2021.

\bibitem{shen2023study}
Kai Shen, Junliang Guo, Xu~Tan, Siliang Tang, Rui Wang, and Jiang Bian.
\newblock A study on relu and softmax in transformer.
\newblock {\em arXiv preprint arXiv:2302.06461}, 2023.

\bibitem{touvron2023llama}
Hugo Touvron, Thibaut Lavril, Gautier Izacard, Xavier Martinet, Marie-Anne Lachaux, Timoth{\'e}e Lacroix, Baptiste Rozi{\`e}re, Naman Goyal, Eric Hambro, Faisal Azhar, et~al.
\newblock Llama: Open and efficient foundation language models.
\newblock {\em arXiv preprint arXiv:2302.13971}, 2023.

\bibitem{rai2024strategies}
Arushi Rai, Kyle Buettner, and Adriana Kovashka.
\newblock Strategies to leverage foundational model knowledge in object affordance grounding.
\newblock In {\em Proceedings of the IEEE/CVF Conference on Computer Vision and Pattern Recognition}, pages 1714--1723, 2024.

\bibitem{godahewa2021monash}
Rakshitha Godahewa, Christoph Bergmeir, Geoffrey~I Webb, Rob~J Hyndman, and Pablo Montero-Manso.
\newblock Monash time series forecasting archive.
\newblock {\em arXiv preprint arXiv:2105.06643}, 2021.

\bibitem{alexandrov2020gluonts}
Alexander Alexandrov, Konstantinos Benidis, Michael Bohlke-Schneider, Valentin Flunkert, Jan Gasthaus, Tim Januschowski, Danielle~C Maddix, Syama Rangapuram, David Salinas, Jasper Schulz, et~al.
\newblock Gluonts: Probabilistic and neural time series modeling in python.
\newblock {\em Journal of Machine Learning Research}, 21(116):1--6, 2020.

\bibitem{wu2021autoformer}
Haixu Wu, Jiehui Xu, Jianmin Wang, and Mingsheng Long.
\newblock Autoformer: Decomposition transformers with auto-correlation for long-term series forecasting.
\newblock {\em Advances in neural information processing systems}, 34:22419--22430, 2021.

\bibitem{lai2018modeling}
Guokun Lai, Wei-Cheng Chang, Yiming Yang, and Hanxiao Liu.
\newblock Modeling long-and short-term temporal patterns with deep neural networks.
\newblock In {\em The 41st international ACM SIGIR conference on research \& development in information retrieval}, pages 95--104, 2018.

\bibitem{liu2023itransformer}
Yong Liu, Tengge Hu, Haoran Zhang, Haixu Wu, Shiyu Wang, Lintao Ma, and Mingsheng Long.
\newblock itransformer: Inverted transformers are effective for time series forecasting.
\newblock {\em arXiv preprint arXiv:2310.06625}, 2023.

\bibitem{nie2022time}
Yuqi Nie, Nam~H Nguyen, Phanwadee Sinthong, and Jayant Kalagnanam.
\newblock A time series is worth 64 words: Long-term forecasting with transformers.
\newblock {\em arXiv preprint arXiv:2211.14730}, 2022.

\bibitem{zhang2023crossformer}
Yunhao Zhang and Junchi Yan.
\newblock Crossformer: Transformer utilizing cross-dimension dependency for multivariate time series forecasting.
\newblock In {\em The Eleventh International Conference on Learning Representations}, 2023.

\bibitem{nichani2024transformers}
Eshaan Nichani, Alex Damian, and Jason~D Lee.
\newblock How transformers learn causal structure with gradient descent.
\newblock {\em arXiv preprint arXiv:2402.14735}, 2024.

\bibitem{sander2024transformers}
Michael~E Sander, Raja Giryes, Taiji Suzuki, Mathieu Blondel, and Gabriel Peyr{\'e}.
\newblock How do transformers perform in-context autoregressive learning?
\newblock {\em arXiv preprint arXiv:2402.05787}, 2024.

\bibitem{wainwright2019high}
Martin~J Wainwright.
\newblock {\em High-dimensional statistics: A non-asymptotic viewpoint}, volume~48.
\newblock Cambridge university press, 2019.

\bibitem{dagan2019learning}
Yuval Dagan, Constantinos Daskalakis, Nishanth Dikkala, and Siddhartha Jayanti.
\newblock Learning from weakly dependent data under dobrushin’s condition.
\newblock In {\em Conference on Learning Theory}, pages 914--928. PMLR, 2019.

\bibitem{kulske2003concentration}
Christof K{\"u}lske.
\newblock Concentration inequalities for functions of gibbs fields with application to diffraction and random gibbs measures.
\newblock {\em Communications in mathematical physics}, 239:29--51, 2003.

\end{thebibliography}
\bibliographystyle{icml2025}

\newpage

\bigskip
\onecolumn

\begin{center}
{\large\bf SUPPLEMENTARY MATERIAL}
\end{center}

\appendix

{
\setlength{\parskip}{-0em}
\startcontents[sections]
\printcontents[sections]{ }{1}{}
}


    

    





        



\section{Table of Notations}
\label{sec:tab_notation}

\begin{table}[h]
    \caption{Mathematical Notations and Symbols}
    \centering
    \begin{tabular}{cl}
    \toprule
        Symbol & Description \\
    \midrule
        $\bx_i$ & The $i$-th component of vector $\bx$ \\
        $\Braket{\ba,\bb}$ & Inner product for vectors $\ba,\bb\in \R^d$ \\
        $[I]$ & Index set $\{1,\cdots,I\}$, where $I\in\mathbb{N}^+$ \\
        $\norm{\cdot}$ & Spectral norm, equivalent to the $l_2$-norm when applied to a vector \\
        $\norm{\cdot}_{2, \infty}$ & The largest L2 norm of column vectors of a matrix \\
        $\bA_{ij}$ & The element on the $i$-th row and $j$-th column of matrix $\bA$ \\
        $\bx_{i:j}$ & The sub-sequence of sequence $\bx$ from coordinate $i$ to $j$ \\
        $\oplus$ & Concatenation between column vectors $\bv \oplus \bu \mapsto ( \bv^\top, \bu^\top)^\top$ \\
        $[\bu ; \bv]$ & Concatenation between two row vectors \\
    \midrule
        $N$ & Length of a transformer input sequence \\
        $T$ & Number of time steps of a time series \\
        $M$ & Number of attention heads. \\
        $q$ & Lag of an $\mathtt{AR}$ process. \\
        $d$ & The number of covariates in an $\mathtt{AR}$ process \\
    \midrule
        $\bv$ & Vector (bold lower) \\
        $\bA$ & Matrix (bold upper) \\
        $\cX$ & random variable (calligraphic) \\
        $\text{x}$ & element from a domain set \\
        $\cD_{\cX}$ & Domain of random variable $\cX$ \\
        $\be_i$ & one-hot vector with its $i$-th entry as 1 \\
    \midrule
        $\mathtt{P}_{\cX}$ & Probability distribution of $\cX$ \\
        $P_{\bz | \bw}(z \mid w )$ & The probability $P \left[ \bz = z \mid \bw = w \right]$ \\
    \bottomrule
    \end{tabular}
     \label{tab:nomenclature}
\end{table}

\clearpage

\section{Related Works}\label{sec:related-work}

\paragraph{Time Series Foundation Models.}
The recent progress in foundation models \cite{touvron2023llama, brown2020language, rai2024strategies} has begun to reshape the field of time series forecasting, a critical task of predicting the future based on history \cite{hamilton2020time}.
However, there are two major challenges in building a time series foundation model:
(a) the model must be able to handle an arbitrary number of covariates, and
(b) the model must generalize to unseen time series domains.
To circumvent (a), several studies simplify the task by considering only univariate time series \cite{ansari2024chronos, rasul2023lag, das2023decoder}.
\cite{das2023decoder} propose a decoder-only transformer pretrained on both real and synthetic datasets.
\cite{rasul2023lag} incorporate lag features and the Llama architecture to pretrain a large uni-variate time series foundation model.
\cite{ansari2024chronos} leverage the power of large language models (LLMs) by using pretrained LLMs backbones.

Recently, \cite{woo2024unified} proposed MOIRAI, the first time series foundation model capable of handling an arbitrary number of covariates.
It addresses (a) by concatenating all covariates into a uni-dimensional sequence, ensuring a consistent input dimension across datasets.
It addresses (b) by pretraining on a large collection of time series datasets \cite{godahewa2021monash, alexandrov2020gluonts, wu2021autoformer, lai2018modeling} spanning domains such as weather, traffic, electricity, and industry.
MOIRAI not only generalizes across a wide range of domains, but its \emph{zero-shot} performance also surpasses several strong supervised learning baselines \cite{liu2023itransformer, nie2022time, zhang2023crossformer}.
However, the machine learning community has yet to provide a suitable explanation for MOIRAI’s impressive performance.
Therefore, this paper is the first to offer theoretical guarantees for MOIRAI as a time series foundation model.

\paragraph{In-Context Learning.}
In-context learning (ICL) is an emerging capability of large foundation models, enabling them to learn diverse and unseen tasks from given examples.
\cite{brown2020language} first provide empirical evidence of ICL in large language models (LLMs); by presenting several examples of 
$(\bx,\by)$ pairs, GPT-3 effectively infers the relationship between $\bx$ and $\by$.
\cite{garg2022can} then conduct quantitative experiments on simple function classes, such as linear regression.
Their results demonstrate that large foundation models can learn the parameters of these function classes.
Subsequently, several theoretical studies \cite{bai2024transformers, von2023transformers, ahn2024transformers, akyrek2023what} have proven that different types of transformers can implement algorithms such as gradient descent.
This discovery provides a theoretical foundation for the empirical findings in \cite{garg2022can}.

The closest studies to this paper are \cite{nichani2024transformers, sander2024transformers}.
However, \cite{sander2024transformers} examines ICL in the context of next-token prediction using a linear transformer.
While their theoretical results relate to in-context learning on sequential data, they are insufficient to explain transformers' success in time series forecasting.
\cite{nichani2024transformers} explores another case where the data is modeled as a Markov chain generated by a transition matrix.
They demonstrate the existence of induction heads that enable transformers to perform next-token prediction.
However, their scenario does not align with multivariate time series, which is where our main contribution lies.

\section{Additional Theoretical Background}\label{sec:additional-theory}
Here, we include several technical lemme that are intensively used throughout our paper.
The Lipschitzness of an MLP layer is obtained in \citep[Lemma~J.1]{bai2024transformers}, which we restate it below
\begin{lemma}[\cite{bai2024transformers}]\label{lem:mlp-lipschitz}
    For a single MLP layer, $\bm{\theta}_2 = (\bW_1, \bW_2)$, we introduce its norm
    \begin{equation*}
        \lvert
        \lvert
        \lvert
        \bm{\theta}_2
        \rvert 
        \rvert 
        \rvert
        =
        \norm{\bW_1}_{\text{op}}
        +
        \norm{\bW_2}_{\text{op}}.
    \end{equation*}
    For any fixed hidden dimension $D^\prime$, we consider
    \begin{equation*}
        \Theta_{2, B}
        \coloneqq
        \{ 
        \bm{\theta}_2
        :
        \lvert
        \lVert
        \bm{\theta}_2
        \rVert
        \rvert
        \leq 
        B
        \}.
    \end{equation*}
    Then for $\bH \in \mathcal{H}_R$, $\bm{\theta}_2 \in \Theta_{2, B}$, the function 
    $(\bm{\theta}_2, \bH) \mapsto \text{MLP}_{\bm{\theta}_2}$ is $(BR)$-Lipschitz w.r.t. $\bm{\theta}_2$
    and $(1 + B^2)$-Lipschitz w.r.t. $\bH$.
\end{lemma}

The following lemma shows any-variate attention is capable of performing variate-wise operation on arbitrary number of covariates under any-variate encoding.

\begin{lemma}[Group-Wise Operation via Any-Variate Attention]\label{lem:moirai-group-wise}

    Let $\norm{\bH}_{2,p} \coloneqq ( \sum_{i=1}^N \norm{\bh_i}_2^p )^{1/p}$ denote the column-wise $(2, p)$-norm of $\bH$.
    For any input matrix $\bH = (\bh_1, \cdots, \bh_T)$ such that $\norm{\bH}_{2, \infty} \leq \mathtt{R}$,
    suppose
    $\psi(\cdot): \R^{D \times T} \rightarrow \R^{D \times T}$ is a sequence-to-sequence function implemented by a single layer standard transformer ($\text{TF}_{\bm{\theta}}^\dagger$) such that 
    \[
    \text{TF}_{\bm{\theta}}^\dagger(\bH)
    \coloneqq
        \psi(\bH).
    \]
    Then there exists a single layer MOIRAI transformer $\text{TF}_{\bm{\theta}}(\cdot)$ such that for any input 
    \[
    \bH^\star
    =
    \begin{bmatrix}
        \bH_1 & \bH_2 & \cdots & \bH_K
    \end{bmatrix},
    \]
    where $\bH_k \in \R^{D \times T}$.
    $\text{TF}_{\bm{\theta}}(\cdot)$ performs
    \[
    \text{TF}_{\bm{\theta}}
    (\bH^\star)
    =
    \begin{bmatrix}
     \psi(\bH_1) & \psi(\bH_2) & \cdots & \psi(\bH_K)   
    \end{bmatrix}.
    \]
\end{lemma}

\begin{proof}[Proof of Lemma~\ref{lem:moirai-group-wise}]\label{proof:group-wise}
    We start by showing the case of a single-head, single-layer standard transformer.
    Let 
    \[
    \text{MLP}_{\bm{\theta}_2} \circ \text{Attn}_{\bm{\theta}}^\dagger
    (\bH)
    =
    \text{MLP}_{\bm{\theta}_2}
    \circ
    \bV \bH
    \sigma
    (
    \Braket{
    \bQ \bH, 
    \bK \bH
    }
    )
    =
    \bV \bH
    \bA_{\bH},
    \]
    where $\bA_{\bH} = \sigma(\Braket{
    \bQ \bH, 
    \bK \bH})$.

    Let $\psi_1(\bH) \coloneqq \bV \bH \bA_{\bH}$, to apply group-wise operation of $\psi_1(\cdot)$ on some input such that
    \[
    \psi_1(
    \bH^\star
    )
    =
    \begin{bmatrix}
        \psi_1(\bH_1) & \psi_1(\bH_2) & \cdots & \psi_1(\bH_K)
    \end{bmatrix}.
    \]

    Let $\bm{0} \in \R^{T \times T}$ be a zero matrix, and $\bm{1} \in \R^{T \times T}$ be a $1$s matrix, for for any input $\norm{\bH^\star}_{2, \infty} \leq \mathtt{R}$, one can find some $u^2 < 0$ to decompose $\psi_1(\cdot)$ into the following form.
    \begin{align*}
    \psi_1(\bH^\star)
    &=
    \bV
    \begin{bmatrix}
        \bH_1 \bA_{\bH_1} & 
        \bH_2 \bA_{\bH_2} &
        \cdots &
        \bH_K \bA_{\bH_K}
    \end{bmatrix}
    \\
    &=
    \bV 
    \bH^\star
    \begin{bmatrix}
        \bA_{\bH_1} & \bm{0} & \cdots  & \bm{0} \\
        \bm{0} & \bA_{\bH_2} & \cdots  & \bm{0} \\
        \vdots & \vdots & \ddots & \vdots \\
        \bm{0} & \bm{0} & \cdots & \bA_{\bH_K}
    \end{bmatrix}
    \\
    &=
    \bV
    \bH^\star
    \times
    \\
    &\sigma
    \left(
    \begin{bmatrix}
        \Braket{\bQ \bH_1, \bK \bH_1} & \Braket{\bQ \bH_1, \bK \bH_2} & \cdots  & \Braket{\bQ \bH_1, \bK \bH_K} \\
        \Braket{\bQ \bH_2, \bK \bH_1} & \Braket{\bQ \bH_2, \bK \bH_2} & \cdots  & \Braket{\bQ \bH_2, \bK \bH_K} \\
        \vdots & \vdots & \ddots & \vdots \\
        \Braket{\bQ \bH_K, \bK \bH_1} & \Braket{\bQ \bH_K, \bK \bH_2} & \cdots & \Braket{\bQ \bH_K, \bK \bH_K}
    \end{bmatrix}
    +
    \begin{bmatrix}
        \bm{0} & u^2 \cdot \bm{1} & \cdots  & u^2 \cdot \bm{1} \\
        u^2 \cdot\bm{1} & \bm{0} & \cdots  & u^2 \cdot \bm{1}\\
        \vdots & \vdots & \ddots & \vdots \\
        u^2 \cdot\bm{1} & u^2\cdot \bm{1} & \cdots & \bm{0}
    \end{bmatrix}
    \right).
    \end{align*}
    Further, observe that operations in an MLP layer are either left multiplication or element-wise operations, which implies group-wise as well. 
    We then finish the proof by setting $u^1 = 0$.

\end{proof}

\begin{theorem}[\text{\citep[Section~5.6]{wainwright2019high}}]\label{thm:generalizd-dudley}
    Suppose $\psi: [0,+\infty_ \rightarrow [0, +\infty)$ is a convex, non-decreasing function satisfying $\psi(x+y) \geq \psi(x) \psi(y)$.
    For any random variable $X$, we consider the Orlicz norm induced by $\psi: \norm{X}_{\psi} \coloneqq \inf \{ K > 0: \mathtt{E}_\psi(|X|/K) \} \leq 1$.
    Suppose that $\{ X_{\theta} \}$ is a zero-mean random process indexed by $\theta\in\Theta$ such that $\norm{ X_{\theta} - X_{\theta^\prime} } \leq \rho(\theta, \theta^\prime)$ for some metric $\rho$ on $\Theta$.
    Then the following holds
    \[
    P
    \left(
    \sup_{\theta, \theta^\prime \in \Theta}
    \vert 
    X_{\theta}
    -
    X_{\theta^\prime}
    \vert 
    \leq 
    8(J+t)
    \right)
    \leq
    \frac{1}{\psi(t/D)},
    \quad
    \text{for all} t \geq 0,
    \]
    where $D$ is the diameter of the metric space $(\Theta, \rho)$, and the generalized Dudley entropy integral $J$ is given by 
    \[
    J
    \coloneqq
    \int_0^D
    \psi^{-1}
    (N(\delta; \Theta, \rho))
    d \delta,
    \]
    where $N(\delta; \Theta, \rho)$ is the $\delta$-covering number of $(\Theta, \rho)$.
\end{theorem}

The next technical lemma is in \cite{bai2024transformers}.
Let $\mathtt{B}^k_{\infty}(R) = [-R, R]^k$ denotes the standard $\ell_{\infty}$ ball in $\R^k$ with radius $R > 0$. 
\begin{definition}[Sufficiently smooth $k$-variable function]
We say a function $g : \R^k \mapsto \R$ is $(R, C_{\ell})$-smooth if for $s = \lceil (k-1)/2 \rceil + 2$, $g$ is a $C^s$ fnction on $\mathtt{B}_{\infty}^k(R)$, and
\[
\sup_{\bz \in \mathtt{B}_{\infty}^k(R)}
\norm{\nabla^i g(\bz)}_{\infty}
=
\sup_{\bz \in \mathtt{B}_{\infty}^k(R)}
\sup_{j_1,\cdots,j_i\in[k]}
| \partial_{x_{j1} \cdots x_{ji}}
g(\bx)
\leq 
L_i
\]
for all $i = 0, 1, \cdots ,s$, with $\max_{0\leq i \leq s} L_i R^i \leq C_{\ell}$.
 
\end{definition}

\begin{lemma}[Approximating smooth $k$-variable functions]\label{lem:approx}
    For any $\varepsilon_{\text{approx}} > 0$, $R \geq 1, C_{\ell} > 0$, we have the following:
    Any $(R, C_{\ell})$-smooth function $g : \R^k \mapsto \R$ is $(\varepsilon_{\text{approx}}, R, M, C)$-approximable by sum of relus with $M \leq C(k) C_{\ell}^2 \log( 1 + C_{\ell}/\varepsilon_{\text{approx}}^2 )$ and $C \leq C(k) C_{\ell}$, where $C(k) > 0$ is a constant that depends only on $k$, i.e.,
    \[
    f(\bz) =
    \sum_{m=1}^M c_m \sigma( \ba_m^\top [\bz ; 1])
    \quad\text{with }
    \sum_{m=1}^M |c_m| \leq C,
    \quad
    \max_{m\in[M]}
    \norm{\ba_m}_1 \leq ,    
    \]
    such that $\sup_{\bz \in [-R,R]^k}
    | f(\bz) - g(\bz) |\leq \varepsilon_{\text{approx}}
    $.
\end{lemma}

\clearpage

\section{Proofs}\label{sec:proofs}
\subsection{Proof of Lemma~\ref{lem:input-causal}}\label{proof:lem-input-casual}
Here we prove a slightly simpler result with the positional encoding containing only zero vectors and a one-hot vector.
One can easily extend the proof by padding the weight matrices.

\begin{equation}\label{eqn:input-data2}
    \bfa H 
    \coloneqq
    \begin{bmatrix}
        x_1 & x_2 &  \dots & x_T & 0
        \\
        \bp_1 & \bp_2 & \dots & \bp_T &
        \bp_{T+1}
    \end{bmatrix}
    \in \R^{D \times (T+1)}
    ,
    \quad
    \bp_i
    \coloneqq
    \begin{bmatrix}
        \mathbf{0}_{d^\prime}
        \\
        \be_i
    \end{bmatrix}
    \in \R^{d^\prime + T}
    ,
\end{equation}

\begin{lemma}[Lemma~\ref{lem:input-causal} Restate]
    Given a sequence of token $\bfa H$ in the form of Equation~\ref{eqn:input-data2}, there exists a one-layer, $q-1$ head ReLU attention layer, such that the columns of $\text{Attn}_{\bm{\theta}}( \bfa H )$ has the following form:
    \begin{equation}
    \text{Attn}_{\bm{\theta}_1}^{\dagger}( \bfa H )_i
    \coloneqq
        \begin{bmatrix}
            x_i
            \\
            x_{i-1}
            \\
            \vdots
            \\
            x_{i-q}
            \\
            \bp_i^\prime
        \end{bmatrix},
        \quad
        \text{where }
        \bp_i^\prime 
        \coloneqq
        \begin{bmatrix}
        \mathbf{0}_{ d^\prime - q }
        \\
        1
        \\
        1 \{ i < T + 1 \}
        \end{bmatrix}
        \in \R^{ d^\prime - q + 2 }.
    \end{equation}
\end{lemma}

\begin{proof}
    Consider an input of the following form
    \begin{equation*}
        \bx
        =
        \begin{bmatrix}
            x_1 & x_2 & \cdots & x_T
            \\
            \bm{0} & \bm{0} & \cdots & \bm{0} 
            \\
            \be_1 & \be_2 & \cdots & \be_T
        \end{bmatrix},
    \end{equation*}
    where $\bx_t \in \R^{d}, \pb_t \in \R^T$, for all $t = 1, \cdots, T$.
    We construct weights of the $m$-th head $\bW_Q^m, \bW_Q^m$ as following,
    \begin{equation*}
        \bW_K^m
        =
        \begin{bmatrix}
            \bm{0}^\top & \bm{0}^\top & \be_1^\top
            \\
            \bm{0}^\top & \bm{0}^\top & \be_2^\top
            \\
            \vdots & \vdots & \vdots 
            \\
            \bm{0}^\top & \bm{0}^\top & \be_T^\top            
        \end{bmatrix},
        \quad
        \bW_Q^m
        =
        \begin{bmatrix}
            \bm{0}^\top & \bm{0}^\top & \be_{1-m}^\top
            \\
            \bm{0}^\top & \bm{0}^\top & \be_{2-m}^\top
            \\
            \vdots & \vdots & \vdots 
            \\
            \bm{0}^\top & \bm{0}^\top & \be_{T-m}^\top            
        \end{bmatrix},
    \end{equation*}
    where we define the negative index as rotational index, i.e., $\be_{-1} = \be_{T}, \be_{-2} = \be_{T-1}$.
    We have
    \begin{align*}
        \left(
        \bW_K^m
        \Xb
        \right)^\top
        \left(
        \bW_Q^m
        \Xb
        \right)
        &=
        \begin{bmatrix}
            \be_1^\top
            \\
            \be_2^\top
            \\
            \vdots 
            \\
            \be_T^\top
        \end{bmatrix}^\top
        \begin{bmatrix}
            \be_{1-m}^\top
            \\
            \be_{2-m}^\top
            \\
            \vdots
            \\
            \be_{T-m}^\top
        \end{bmatrix}
        \\
        &=
        \bI_T
        \begin{bmatrix}
            \be_{1-m}
            \\
            \be_{2-m}
            \\
            \vdots
            \\
            \be_{T-m}
        \end{bmatrix}.
    \end{align*}
    Note that the result of 
        $\sigma\left(\left(
        \bW_K^m
        \Xb
        \right)^\top
        \left(
        \bW_Q^m
        \Xb
        \right)
        \right)$
    is a rotation matrix, where right multiplication on $\Xb$ will rotate the columns of $\Xb$.
    Therefore, we have 
    $\bW_V^m$ that performs row-wise shifting and the attention matrix 
    $\sigma\left(\left(
        \bW_K^m
        \Xb
        \right)^\top
        \left(
        \bW_Q^m
        \Xb
        \right)
        \right)$
        performs column-wise shifting.
\end{proof}

\subsection{Proof of \cref{thm:any-variate-auto}}\label{proof:any-var-enc}

\paragraph{Autoregressive Linear Regression under Any-Variate Encoding.}
The ultimate goal of this setup is to perform the following mechanism.
Let $\bx$ be the target variate we wish to predict, $\bz^j$ be the $j$-th covariate of $\bx$, for $j \in [M]$.
We denote the lookback window size as $q$, and each covariate has length $T$ ($T$-time steps.).
We denote the time encoding as $\bp_i$ for $i \in [T]$, and the variate encoding as $\bq_{j}$ for $j \in [M]$.
Finally, our goal is to predict $\bx_T$.

\begin{equation*}
\begin{bmatrix}
    x_1^1 & \cdots & x_T^1 & x_1^2 & \cdots & x_T^2 & \cdots & x_1^d & \cdots & x_T^d
    \\
    \bp_1 & \cdots & \bp_T & \bp_1 & \cdots & \bp_T & \cdots & \bp_1 & \cdots & \bp_T
    \\
    \be_1 & \cdots & \be_1 & \be_2 & \cdots & \be_2 & \cdots & \be_d & \cdots & \be_d
\end{bmatrix}
\mapsto
\begin{bmatrix}
    \mathtt{A}_1(q) & \cdots
    \\
    \mathtt{A}_2(q) & \cdots
    \\
    \vdots
    \\
    \mathtt{A}_d(q) & \cdots
    \\
    \vdots & \ddots
\end{bmatrix}.
\end{equation*}

Here, different colors represent different covariates.
The motivation for performing such an operation is to apply the in-context learning property of transformers proved in \cite{bai2024transformers}.

\begin{lemma}[Lemma~\ref{lem:mar-group-wise} Restate]
Define the matrix $\mathtt{A}_i(q)$ for the $i$-th covariates $(x_1^i, \cdots, x_T^i)$, with order $q$, such that
\begin{equation*}
    \mathtt{A}_i(q)
    \coloneqq
    \begin{bmatrix}
        x_1^i & x_2^i & \cdots & x_t^i & x_{t+1}^i & x_{t+2}^i & \cdots \\
        x_T^i & x_{T-1}^i & \cdots & x_{t-1}^i & x_t^i & x_{t+1}^i & \cdots \\
        x_{T-1}^i & x_{T-2}^i & \cdots & x_{t-2}^i & x_{t-1}^i & x_{t}^i & \cdots \\
        \vdots & \vdots & \vdots & \vdots & \vdots & \vdots & \cdots \\
        x_{T-q}^i & x_{T-q+1}^i & \cdots & x_{t-q}^i & x_{t-q+1}^i & x_{t-q+2}^i & \cdots \\   
    \end{bmatrix},
\end{equation*}
where in the $j$-th column of $\mathtt{A}_i(q)$, it contains historical values of $x_j^i$ with lag $q$.

Given fixed $D, T \in \N^+$, where $T > q$.
For any input matrix $\bH$ in the form of Any-Variate Encoding in Equation~\ref{eqn:AV-encoding}, such that $\bH \in \R^{ D^\prime \times dT^\prime }$, and $D^\prime \leq D$, $T^\prime < T$.
There exists a 1-layer, $q$ head Any-Variate Attention that performs the following operation.
    \begin{equation*}
    \begin{bmatrix}
        x_1^1 & \cdots & x_T^1 & x_1^2 & \cdots & x_T^2 & \cdots & x_1^d & \cdots & x_T^d
        \\
        \bp_1 & \cdots & \bp_T & \bp_1 & \cdots & \bp_T & \cdots & \bp_1 & \cdots & \bp_T
        \\
        \be_1 & \cdots & \be_1 & \be_2 & \cdots & \be_2 & \cdots & \be_d & \cdots & \be_d
    \end{bmatrix}
    \mapsto
    \begin{bmatrix}
        \mathtt{A}_1(q) & \mathtt{A}_2(q) & \cdots & \mathtt{A}_d(q)
        \\
        \ddots & \ddots & \cdots & \ddots 
    \end{bmatrix}
    \end{equation*}
\end{lemma}

\begin{proof}
    The proof comes as a direct corollary of Lemma~\ref{lem:moirai-group-wise} and \citep[Proposition~A.5]{bai2024transformers}.
    By Lemma~\ref{lem:input-causal}, there exists a single layer standard transformer layer (with $\bW_1, \bW_2$ being $0$s) that generates $\mathtt{A}_i(q)$ for each uni-variate (covariate).
    It then left applying Lemma~\ref{lem:moirai-group-wise} for variate-wise operation and applying \citep[Proposition~A.5]{bai2024transformers} to keep the time indices $\bp_t$ unchanged.
        
\end{proof}

\begin{corollary}\label{cor:second-step-enc}
    There exists a $d_{\max}$ head standard attention layer that performs the following
    \[
        \begin{bmatrix}
            \mathtt{A}_1(q) & \mathtt{A}_2(q) & \cdots & \mathtt{A}_d(q)
            \\
            \ddots & \ddots & \cdots & \ddots 
        \end{bmatrix}
        \mapsto
        \begin{bmatrix}
            \mathtt{A}_1(q) & \cdots 
            \\
            \tilde{\mathtt{A}}_2(q) & \cdots
            \\
            \vdots
            \\
            \tilde{\mathtt{A}}_d(q) & \cdots
            \\
            \ddots & \ddots 
        \end{bmatrix},
        \quad\text{for any }d \leq d_{\max},
    \]
    where $\tilde{\mathtt{A}}_i(q)$ is $\mathtt{A}_i(q)$ without the first row.
\end{corollary}
\begin{proof}
Note that this operation in \cref{cor:second-step-enc} is straightforward with \cref{lem:input-causal} and \citep[Proposition~A.5]{bai2024transformers}.
As for each $i \in [d]$, $i \neq 1$, the attention layer performs two operations to each element of $\mathtt{A}_i(q)$:

\begin{align*}
    \begin{cases}
        iT \text{ columns to the left } & \text{ right multiplication} \\
        q_{\max} \text{ rows below } & \text{ left multiplication} \\
        \text{zero out} & \text{if in first row (left multiplication)}
    \end{cases}.
\end{align*}
Note that one can simply construct weight matrices to perform the above permutations and masking.
In total, we need $d_{\max}$ heads to perform such operations for each $\mathtt{A}_i(q)$, for any $d \leq d_{\max}$.
For $q < q_{\max}$, the remaining entries will be zero padded.
Finally, with at best $2$ layers of $d_{\max}$ head any-variate attention, we then obtain
\[
\tilde{H}^{(2)}
\coloneqq
    \begin{bmatrix}
        \mathtt{A}_1(q) & \cdots 
        \\
        \tilde{\mathtt{A}}_2(q) & \cdots
        \\
        \vdots
        \\
        \tilde{\mathtt{A}}_d(q) & \cdots
        \\
        \bp
        \\
        \be
    \end{bmatrix}
    =
    \begin{bmatrix}
     \cdots & x_{T-1}^1 & \mathcolor{red}{x_T^1} &\cdots 
    \\
    \cdots & x_{T-2}^1 & \mathcolor{blue}{x_{T-1}^1} &\cdots
    \\
    \vdots & \vdots & \mathcolor{blue}\vdots &\cdots
    \\
    \cdots & x_{T-q}^1 & \mathcolor{blue}{x_{T-q}^1} &\cdots
    \\
    \cdots & x_{T-1}^d & \mathcolor{blue}{x_{T-1}^d}
    \\
    \cdots & x_{T-2}^d & \mathcolor{blue}{x_{T-2}^d}
    \\
    \cdots & \vdots & \mathcolor{blue}\vdots \\
    \\
    \cdots & x_{T-q}^d & \mathcolor{blue}{x_{T-q}^d}
    \\
    \cdots & \bp_{T-1} & \bp_{T}
    \\
    \cdots & \be_1 & \be_1
    \end{bmatrix},
\]
where $\bp$ is the matrix of $(\bp_1, \cdots, \bp_T)$, $\be$ is the matrix of $(\be_1, \cdots, \be_1)$.

Note that $x_T^1$ in red is the target we wish to predict (masked as 0 initially), and the entries in blue is considered the input feature of our AR model (a linear regression model in this case),
and we are able to directly apply several theoretical results in \cite{bai2024transformers} with input $\tilde{\bH}^{(2)}$.
Specifically, for \cref{thm:any-variate-auto}, it follows directly from \citep[Theorem~4]{bai2024transformers} by setting $\lambda = 0$.

\end{proof}

Next, we present several approximation results from \cite{bai2024transformers}, which our approximation results follows immediately from.
Considering the general form of autoregressive data:
$\bx \in \R^{d \times T} \coloneqq (\bx_1, \dots, \bx_T)$, where $\bx_t = (x_t^1, \cdots, x_t^d) \in \R^d$.
Assuming our target (variate of interest) is in dimension $1$, we assume the autoregressive process generates $x_t^1$ as follows:
\begin{equation}
    x_{t}^1
    =
    f( \bx_{t-q: t-1}^{1:d})
    +
    \epsilon_t
    ,
\end{equation}
where $\epsilon_t \sim N(0, \sigma^2)$, $a_i^j \in \R^1$, and $f$ is a function of interest.
We then present several results when $f$ varies.

\paragraph{Non-Linear AR.}
Here we analyze that when the autoregressive process is generated by a 2 layer ReLU network with look back window size $q$.
Suppose the prediction function $\text{pred}(\bx, \bw) \coloneqq \sum_{k=1}^K u_k r ( \bv_k^\top \bx)$ is given by a two-layer neural network, parameterized by $\bw = [\bw_k, u_k]_{k\in[K]} \in \R^{K(d+1)}$.
Consider the ERM problem:
\[
\min_{\bw \in \cW}
\hat{L}_N(\bw)
\coloneq
\frac{1}{2N}
\sum_{i=1}^N
\ell( \text{pred}(\bx_i, \bw), y_i )
=
\frac{1}{2N}
\sum_{i=1}^N
\ell
\left(
\sum_{k=1}^K
u_k r(\bv_k^\top \tilde{\bx}_i), x_T^1 
\right),
\]
where $\cW$ is a bounded domain and $\tilde{\bx}_i \in \R^{qd}$ is a flatten version of $\bx_{t-q:t-q} \in \R^{d\times q}$.
\begin{proposition}
    Fix any $B_w, B_u > 0$, $L \geq 3, \nu > 0$, and $\varepsilon > 0$.
    Suppose that
    \begin{enumerate}
        \item Both the activation function $r$ and the loss function $\ell$ is $C^4$-smooth.
        \item $\cW$ is a closed domain such that $\cW \subset \left\{ \bw = [\bv_k;u_k]_{k\in[K]} \in \R^{K(d+1)} : \norm{\bv_k}_2 \leq B_v, |u_K| \leq B_u \right\}$, and $\text{Proj}_{\cW} = \text{MLP}_{\bm{\theta_2}}$ for some MLP layer with hidden dimension $D_w$ and $\norm{\bm{\theta}_2}_{\text{op}} \leq C_w$.
    \end{enumerate}
    Then there exists a $(L_1 + 2 L_2)$-layer MOIRAI transformer with 
    \begin{align*}
    \max_{\ell\in[L_1+1,2L_2]}
    M^{(\ell)} \leq \tilde{O}(\varepsilon^{-2}), \quad
    \max_{\ell\in[L_1+1,2L_2]}
    D^{(\ell)} \leq \tilde{O}(\varepsilon^{-2}) + D_w, \quad
    \\
    \norm{\bm{\theta}}_{\text{op}} \leq O(1+\eta) + C_w, \quad
    \sum_{\ell=1}^{L_1} M^{(\ell)} = d_{\max} + q_{\max}.
    \end{align*}
    where we hide the constants $K, B_x, B_u, B_v, C^4$,
    satisfies the following
    \[
    \norm{ \hat{\bw} - \bw^{L}_{\text{GD}} }_2
    \leq 
    L_f^{-1} ( 1 + \eta L_f)^L \varepsilon,
    \]
    where $L_f = \sup_{w\in\cW} \norm{ \nabla^2 \hat{L}_N(\bw) }_2$.
    
\end{proposition}

\paragraph{Maximum Likelihood Estimation (Gaussian) via Transformers.}
The next result shows that MOIRAI transformers are also capable of performing maximum likelihood estimation on any input multi-variate time series.
Given a data generated by some $\mathtt{AR}_d(q)$ process with parameter $(\bw_1, \cdots, \bw_q) \subset \R^{d}$: 
$(\bx_1, \cdots, \bx_T) \subset \R^d$, the conditional likelihood $f(\cdot)$ of observing $\bx_t$ is
\[
f(\bx_t \mid \bx_{t-1}, \cdots, \bx_{t-q} )
\coloneq
\frac{1}{\sqrt{2\pi \sigma^2}}
\exp
\left(
\frac{-( \bx_t - \sum_{i=1}^q \langle \bw_i,  \bx_{t-i} \rangle )^2 }{2 \sigma^2}
\right).
\]
The goal is to estimate the mean vector $(\bw_1, \cdots, \bw_q)$ and the variance $\sigma^2$ by minimizing the negative log-likelihood loss.
Note that with $n \geq d$, the loss is strongly convex.
The optimization over the NLL Loss has two steps: estimating the mean vector: $\hat{\bw}$, and then derive the variance $\hat{\sigma}^2$ with the following closed-form solution:
\[
\sigma^2
=
\frac{1}{T}
\sum_{t=1}^T
\left(
\bx_t - \sum_{i=1}^q
\langle \hat{\bw}_i, \bx_{t-i} \rangle
\right)^2.
\]

\begin{theorem}\label{thm:moirai-mle}
    Given a set of input data generated by some $\mathtt{AR}_d(q)$ process: $\bX, \in \R^{n \times d}, \bY \in \R^n$, considering the following negative log-likelihood loss, the goal is to find a set of parameters $\bw \in \R^d, \sigma^2 \in \R^+$ to minimize the following loss
    \[
    L_{\text{NLL}}
    ( \bw, \sigma )
    \coloneq
    \frac{n}{2}
    \log (2 \pi \sigma^2) +
    \frac{1}{2\sigma^2}
    \sum_{t=q+1}^T
    \left(
    \bx_t - \sum_{i=1}^q
    \langle \bw_i, \bx_{t-i} \rangle
    \right)^2
    \]
    We denote $\bw^\star, \sigma^\star$ as the ERM satisfying the NLL Loss.
    There exists a $(L_1 + L_2 + 2)$-layer MOIRAI Transformer such that its first $L_1+L_2$ layers follow the same requirement in \cref{thm:any-variate-auto}, and the last two layers each has two and one heads, it estimates $\bw, \bm{\sigma}$ with bounded error:
    \[
    \norm{ \hat{\bw} - \bw^\star } \leq
    \varepsilon,
    \]
    and the estimated variance is bounded by
    \[
    \left\vert \hat{\sigma}^2 - {\sigma^\star}^2 \right\vert
    \leq
    2 E B_x \varepsilon + B_x^2 \varepsilon
    =
    \tilde{O}(\varepsilon + \varepsilon^2),
    \]
    where $E \leq B(1+B_w)$, and $\tilde{O}$ hides the values dependent on $B_x, B_w$.
\end{theorem}

\begin{proof}[Proof of \cref{thm:moirai-mle}]
    
\[
L_{\text{NLL}}
( \bw, \sigma )
\coloneq
\frac{n}{2}
\log (2 \pi \sigma^2) +
\frac{1}{2\sigma^2}
\sum_{t=1}^T
\left(
\bx_t - \sum_{i=1}^q
\langle \bw_i, \bx_{t-i} \rangle
\right)^2.
\]
Following \cref{thm:any-variate-auto}, the first $L_1 + L_2$ layers of MOIRAI obtains $\hat{\bw}$ such that the $L_1+L_2+1$-th layer takes the following as input
\begin{align*}
    \tilde{\bh}_i^{(L_1+L_2)}
    =    \left[
    x^1_i ; 
    \bx_{i-1:i-q}^1
    ;
    \bx_{i-1:i-q}^2
    ;
    \cdots 
    ;
    \bx_{i-1:i-q}^d
    ; 
    \bw^\star + \bm{\varepsilon}
    ; \bm{0} ; 1; t_i
    \right],
\end{align*}
where $\bw^\star + \bm{\varepsilon} \in \R^{qd}$ is the flatten mean vectors.
For the simplicity of notations, for the $i$-th column, we denote $x^1_i$ with $\tilde{\by}_i$, and denote $[\bx_{i-1:i-q}^1
    ;
    \bx_{i-1:i-q}^2
    ;
    \cdots 
    ;
    \bx_{i-1:i-q}^d]$ as $\tilde{\bx}_i \in \R^{qd}$, as they correspond to the label and feature of our AR model, respectively.
$\bm{\varepsilon} \in \R^{dq}$ satisfies 
\[
\norm{\bm{\varepsilon}}
\leq
\varepsilon \cdot (\eta B_x).
\]

Now we start to construct the $(L_1+L_2+1)$-th layer.
One can then construct
\begin{align*}
    \bQ_1^{L+1} \bh^L_i &= [ \bm{0} ; \tilde{\bx}_i ; \bm{0} ], \quad \bK^{L+1}_1 \bh^L_j = [ \bm{0} ; \hat{\bw} ; \bm{0} ], \quad \bV^{L+1}_1 \bh_k^L = [   \bm{0} ; 1 ; \bm{0} ]
    \\
    \bQ_2^{L+1} \bh^L_i &= [ \bm{0} ; \tilde{\bx}_i ; \bm{0} ], \quad \bK^{L+1}_2 \bh^L_j = [ \bm{0} ; -\hat{\bw} ; \bm{0} ], \quad \bV^{L+1}_2 \bh_k^L = [ \bm{0} ; -1 ; \bm{0} ].
\end{align*}

The above construction gives us
\begin{align*}
    \bh_i^{L+1}
    &=
    \bh_i^{L}
    +
    \frac{1}{n}
    \sum_{j=1}^{n}
    \sum_{m=1}^2
    \sigma\left(
    \langle 
    \bQ^{L+1}_m \bh_{n+1}^L,
    \bK_m^{L+1}
    \bh_{j}^L
    \rangle
    \right)
    \bV_m^{L+1}
    \bh_j^L
    \\
    &=
    [ \tilde{\by}_i ; \tilde{\bx}_i; \hat{\bw} ; \bm{0} ; 1 ; t_i]
    +
    \left(
    \sigma
    (
    \langle \hat{\bw}, \bx_i \rangle
    )
    -
    \sigma 
    (
    -
    \langle \hat{\bw}, \bx_i \rangle
    )
    \right)
    \cdot [ \bm{0} ; 1 ; \bm{0} ]
    \\
    &=
    [ \tilde{\by}_i ; \tilde{\bx}_i; \hat{\bw} ;  \langle \hat{\bw} , \tilde{\bx_i} \rangle,  \bm{0} ; 1 ; t_i].
\end{align*}

Next, we construct the last layer as
\begin{align*}
    \bQ_1^{L+1} \bh^L_i &= [ ... ; \tilde{\by}_i -  \langle \hat{\bw}, \tilde{\bx}_i \rangle  ; ... ], \quad \bK^{L+1}_1 \bh^L_j = [ ... ; \tilde{\by}_j -  \langle \hat{\bw}, \tilde{\bx}_j \rangle ;  ... ], \quad \bV^{L+1}_1 \bh_k^L = [ \bm{0} ; 1 ; \bm{0} ]
\end{align*}

Finally, the result becomes 
\begin{align*}
    \bh_{i}
    =
    \frac{1}{n}
    [ ... ; \sum_{\mu=1}^n ( \by_\mu - \langle \bx_\mu , \hat{\bw} \rangle )^2; ... ]
    =
    [ ... ; \hat{\sigma^2}; ... ].
\end{align*}

Thus, we complete the proof.
\end{proof}

\clearpage

\subsection{Proof of the Lipschitzness of Any-Variate Transformers}\label{proof:tr-lipschitz}

We first show the Lipschitzness of each component in an Any-Variate Transformer.
For any $p \in [1, \infty]$, let $\norm{\bH}_{2,p} \coloneqq ( \sum_{i=1}^N \norm{\bh_i}_2^p )^{1/p}$ denote the column-wise $(2, p)$-norm of $\bH$.
For any radius $\mathtt{R} > 0$, we denote 
$\cH_{\mathtt{R}} \coloneqq \{ \bH : \norm{\bH}_{2, \infty} \leq \mathtt{R} \} $ be the ball of radius $\mathtt{R}$ under norm $\norm{\cdot}_{2, \infty}$.

\begin{lemma}\label{lem:attn-lipschitz}
    For a single Any-Variate attention layer, $\bm{\theta}_1 = \{(\bV_m, \bQ_m, \bK_m, u^1_m, u^2_m)\}_{m\in[M]}$, we introduce its norm
    \begin{equation*}
        \Vert
        \bm{\theta}_1
        \Vert
        \coloneqq
        \max_{m\in[M]}
        \{ 
            \norm{\bQ_m}_{\text{op}},
            \norm{\bK_m}_{\text{op}},
            \vert u^1_m \vert,
            \vert u^2_m \vert
        \}
        +
        \sum_{m=1}^M \norm{\bV_m}_{\text{op}}
    \end{equation*}
    For any fixed hidden dimension $D^\prime$, we consider
    \begin{equation*}
        \Theta_{1, B}
        \coloneqq
        \{ 
        \bm{\theta}_{1}
        :
        \lvert \lvert \lvert
        \bm{\theta}_1
        \rvert \rvert \rvert
        \leq 
        B
        \}.
    \end{equation*}
    Then for $\bH \in \cH_{\mathtt{R}}$, $\bm{\theta}_1 \in \Theta_{1, B}$, the function 
    $(\bm{\theta}_1, \bH) \mapsto \text{Attn}_{ \bm{\theta}_1}$ is $(1+\iota)$-Lipschitz w.r.t. $\bm{\theta}_1$, where $\iota = \max \{  B^2\mathtt{R}^2 +T + (T-1)d, B(T-1)d  \}$,
    and $(1 + B^3 \mathtt{R}^2)$-Lipschitz w.r.t. $\bH$.
\end{lemma}

\begin{proof}
    Given some $\epsilon > 0$, some set $X$ and a function class $\mathcal{F}$.
    If $\mathcal{F}$ is $L$-Lipschitzness, i.e., 
    \begin{equation*}
        \norm{f(x_1) - f(x_2)} \leq
        L \norm{x_1 - x_2}, \quad \text{for all }f \in \mathcal{F}.
    \end{equation*}
    Then, the following holds
    \begin{equation*}
        N(\epsilon, \mathcal{F}, \norm{\cdot})
        \leq
        N( \nicefrac{\epsilon}{L}, X, \norm{\cdot} ).
    \end{equation*}

    Define 
    \begin{equation*}
        \Theta_{\text{attn}, B}
        \coloneqq
        \{ \theta_{\text{attn}}:  \lvert \lVert \theta_{\text{attn}} \rVert \rvert  \leq B \}.
    \end{equation*}

    The output of the Any-Variate Attention $[\tilde{h}_i]$ is given by
    \begin{equation*}
        \tilde{h}_i
        =
        h_i
        +
        \sum_{m=1}^M
        \frac{1}{N}
        \sum_{j=1}^N
        \sigma
        \left(
        \Braket{ \Qb_m h_i, \Kb_m h_j }
        \cdot \Vb_m h_j
        +
        u_m^1
        \star 
        \Ub
        +
        u_m^2
        \star
        \bar{\Ub}
        \right).
    \end{equation*}
    We also define $\theta^\prime_{\text{attn}} = \{ \Vb_m^\prime, \Qb_m^\prime, \Kb_m^\prime, u_m^{1\prime}, u_m^{2 \prime} \}_{m \in [M]}$.
    $\tilde{h}^\prime_i$ as
    \begin{equation*}
        \tilde{h}_i^\prime
        =
        h_i
        +
        \sum_{m=1}^M
        \frac{1}{N}
        \sum_{j=1}^N
        \sigma
        \left(
        \Braket{ \Qb_m^\prime h_i, \Kb_m^\prime h_j }
        \cdot \Vb_m^\prime h_j
        +
        u_m^{1\prime}
        \star 
        \Ub
        +
        u_m^{2 \prime}
        \star
        \bar{\Ub}
        \right).
    \end{equation*}

    Now we bound 
    $
    \norm{
    \text{Attn}_{\bm{\theta}_1}( \bH )
    -
    \text{Attn}_{\bm{\theta}_1^\prime}(\bH)
    }_{2, \infty}
    =
    \max_i 
    \norm{ \tilde{\bh}_i - \tilde{\bh}_i^\prime }_2 
    $ as follows
    \begin{align*}
        \norm{ \tilde{\bh}_i - \tilde{\bh}_i^\prime }_2 
        &=
        \Bigg\Vert
        \sum_{m=1}^M
        \frac{1}{N}
        \Bigg[
        \sum_{j=1}^N
        \sigma
        \left(
        \Braket{ \Qb_m h_i, \Kb_m h_j }
        +
        u^1_m
        \star 
        \Ub
        +
        u^2_m
        \star
        \bar{\Ub}
        \right)
        \Vb_m h_j
        -
        \\
        &\quad\quad
        \sum_{j=1}^N
        \sigma
        \left(
        \Braket{ \Qb_m^\prime h_i, \Kb_m^\prime h_j }
        +
        u_m^{1\prime}
        \star 
        \Ub
        +
        u_m^{2 \prime}
        \star
        \bar{\Ub}
        \right)
        \Vb_m^\prime h_j
        \Bigg]
        \Bigg\Vert_2
        \\
        &\leq
        \sum_{m=1}^M
        \frac{1}{N}
        \sum_{j=1}^N
        \Big\Vert
        \sigma
        \left(
        \Braket{ \Qb_m h_i, \Kb_m h_j }
        +
        u^1_m
        \star 
        \Ub
        +
        u^2_m
        \star
        \bar{\Ub}
        \right)
        \Vb_m h_j 
        \\
        &\quad\quad
        -
        \sigma
        \left(
        \Braket{ \Qb_m^\prime h_i, \Kb_m^\prime h_j }
        +
        u_m^{1\prime}
        \star 
        \Ub
        +
        u_m^{2 \prime}
        \star
        \bar{\Ub}
        \right)
        \Vb_m^\prime h_j
        \Big\Vert_2
        \\
        &\leq
        \sum_{m=1}^M
        \frac{1}{N}
        \sum_{j=1}^N
        \norm{h_j}_2
        \Big\Vert
        \sigma
        \left(
        \Braket{ \Qb_m h_i, \Kb_m h_j }
        +
        u_m^1
        \star 
        \Ub
        +
        u_m^2
        \star
        \bar{\Ub}
        \right)
        \Vb_m
        \\
        &\quad\quad
        -
        \sigma
        \left(
        \Braket{ \Qb_m^\prime h_i, \Kb_m^\prime h_j }
        +
        u_m^{1\prime}
        \star 
        \Ub
        +
        u_m^{2 \prime}
        \star
        \bar{\Ub}
        \right)
        \Vb_m^\prime
        \Big\Vert_{\text{op}}.
    \end{align*}

    Let 
    \begin{align*}
        A &= 
        \Braket{ \Qb_m h_i, \Kb_m h_j }
        +
        u^1_m
        \star 
        \Ub
        +
        u_m^2
        \star
        \bar{\Ub}
        \\
        B &= 
        \Braket{ \Qb_m^\prime h_i, \Kb_m^\prime h_j }
        +
        u_m^{1\prime}
        \star 
        \Ub
        +
        u_m^{2 \prime}
        \star
        \bar{\Ub}.
    \end{align*}

    By triangle inequality, we have
    \[
    \norm{
    \sigma(A) V_m - \sigma(B) V_m^\prime
    }
    \leq 
    \norm{\sigma(A)}_{\text{op}}
    \norm{\bV_m - \bV_m^\prime}_{\text{op}}
    +
    \norm{\sigma(A) - \sigma(B)}_{\text{op}}
    \norm{V^\prime_m}_{\text{op}}.
    \]

    Note that $\sigma(\cdot)$ is $1$-Lipschitz, we get
    \begin{align*}
        \norm{\sigma(A) - \sigma(B)}_{\text{op}}
        &\leq
        \norm{A-B}_{\text{op}}
        \\
        &=
        \norm{
        \Braket{ \Qb_m h_i, \Kb_m h_j } 
        -
        \Braket{ \Qb_m^\prime h_i, \Kb_m^\prime h_j }
        +
        ( u^1_m - u_m^{1\prime} ) \bU
        +
        ( u^2_m - u_m^{2 \prime} ) \hat{\bU}
        }_{\text{op}}
        \\
        &\leq
        \norm{
        \big\vert
            \Braket{ \Qb_m h_i, \Kb_m h_j } 
        -
        \Braket{ \Qb_m^\prime h_i, \Kb_m^\prime h_j }
        \big\vert
        +
        \bigg\vert
        \norm{
        ( u^1_m - u_m^{1\prime} ) \star \bU
        }
        \bigg\vert
        +
        \bigg\vert
        \norm{
        ( u^2_m - u_m^{2 \prime} ) \star \bar{\bU}
        }
        \bigg\vert
        }
        .
    \end{align*}

    For the first term in the last inequality, we have
    \begin{align*}
    \Braket{ \Qb_m h_i, \Kb_m h_j } 
        -
    \Braket{ \Qb_m^\prime h_i, \Kb_m^\prime h_j }
    &\leq
    \norm{
    \bQ_m - \bQ_m^\prime
    }
    \norm{h_i}
    \norm{h_j}
    \norm{\bK_m}
    +
    \norm{
    \bK_m - \bK_m^\prime
    }
    \norm{h_i}
    \norm{h_j}
    \norm{\bQ_m}
    \\
    &=
    \mathtt{R}^2 B
    \left(
    \norm{
    \bQ_m - \bQ_m^\prime
    }
    +
    \norm{
    \bK_m - \bK_m^\prime
    }
    \right).
    \end{align*}
    Further, we have
    \[
    \norm{ (u^1_m - u_m^{1\prime}) \star \bU }
    \leq
    \vert 
    u^1_m - u_m^{1\prime}
    \vert 
    \norm{\bU}
    \leq
    T    \vert 
    u^1_m - u_m^{1\prime}
    \vert ,
    \]
    where $T$ is the length of each variate (lookback window size).
    \[
    \norm{ (u^2_m - u_m^{2 \prime}) \star \bar{\bU} }
    \leq
    \vert 
    u^2_m - u_m^{2 \prime}
    \vert 
    \norm{\bar{\bU}}
    \leq
    (T-1)d
    \vert 
    u^2_m - u_m^{2 \prime}
    \vert 
    ,
    \]
    where $d$ is the number of variates.
    
    Thus, we have
    \begin{align*}
        \norm{\sigma(A) - \sigma(B)}_{\text{op}}
        \norm{\bV_m^\prime}_{\text{op}}
        &\leq
        B 
        \left( 
        \mathtt{R}^2 B
        \left(
        \norm{
        \bQ_m - \bQ_m^\prime
        }
        +
        \norm{
        \bK_m - \bK_m^\prime
        }
        \right)
        +
        T \left( \vert u^1_m - u_m^{1\prime} \vert  \right)
        +
        (T-1)d
        \left(
        \vert 
        u^2_m - u_m^{2 \prime}
        \vert
        \right)
        \right)
        \\
        &\leq
        B \cdot \max \{ \mathtt{R}^2 B, (T-1) d \}
        \mathcolor{red} \cdot
        \left(
        \norm{\bQ_m - \bQ^\prime_m}
        +
        \norm{\bK_m - \bK^\prime_m}
        +
        \vert u^1_m - u_m^{1\prime} \vert 
        +
        \vert u^2_m - u_m^{2 \prime} \vert 
        \right)
        .
    \end{align*}
    Next, we bound
    \begin{align*}
        \norm{\sigma(A)}_{\text{op}}
        \leq
        \norm{
        A
        }_{\text{op}}
        \leq
        B^2 \mathtt{R}^2 + ( T + (T-1)d ),
    \end{align*}
    due to the fact that 
    \[
    \norm{A}
    \leq
    \norm{
    \bQ_m h_i
    }
    \norm{
    \bK_m h_j
    }
    \norm{
    u^1_m \bU
    }
    \norm{
    u^2_m \bar{\bU}
    }.
    \]
    Overall, the Any-Variate Attention is 
    $
    \max \{  B^2\mathtt{R}^2 +T + (T-1)d, B(T-1)d  \}
    $-Lipshcitz in $\bm{\theta}_1$.
\end{proof}

\begin{proof}
    We start by considering $\bH^\prime = [\bh_i^\prime]$ and
    \begin{equation*}
        \tilde{\bh}_i^\prime
        =
        \bh^\prime_i
        +
        \sum_{m=1}^M
        \frac{1}{N}
        \sum_{j=1}^N
        \sigma
        \left(
        \langle
        \Qb_m \bh^\prime_i,
        \Kb_m \bh^\prime_j
        \rangle
        +
        u^1_m \cdot \Ub
        +
        u^2_m \cdot \bar{\Ub}
        \right)
        \cdot 
        \Vb_m \bh_j^\prime.
    \end{equation*}
    We then bound 
    \begin{align*}
        &\norm{ (\tilde{\bh}_i^\prime - \bh_i^\prime)
        -
        ( \tilde{\bh}_i - \bh_i )
        }_2
        \\
        &=
        \norm{
        \sum_{m=1}^M
        \frac{1}{N}
        \sum_{j=1}^N
        \left[
        \sigma
        \left(
        \langle
        \Qb_m \bh^\prime_i,
        \Kb_m \bh^\prime_j
        \rangle
        +
        u^1_m \cdot \Ub
        +
        u^2_m \cdot \bar{\Ub}
        \right)
        \Vb_m \bh_j
        -
        \left(
        \langle
        \Qb_m \bh^\prime_i,
        \Kb_m \bh^\prime_j
        \rangle
        +
        u^1_m \cdot \Ub
        +
        u^2_m \cdot \bar{\Ub}
        \right)
        \Vb_m \bh_j^\prime
        \right]
        }_2
        \\
        &\leq
        \sum_{m=1}^M
        \frac{1}{N}
        \sum_{j=1}^N
        \norm{\Vb_m}_{\text{op}}
        \norm{
        \sigma
        \left(
        \langle
        \Qb_m \bh^\prime_i,
        \Kb_m \bh^\prime_j
        \rangle
        +
        u^1_m \cdot \Ub
        +
        u^2_m \cdot \bar{\Ub}
        \right)
         \bh_j
        -
        \left(
        \langle
        \Qb_m \bh^\prime_i,
        \Kb_m \bh^\prime_j
        \rangle
        +
        u^1_m \cdot \Ub
        +
        u^2_m \cdot \bar{\Ub}
        \right)
        \bh_j^\prime
        }_2
        \\
        &\leq
        \sum_{m=1}^M 
        \frac{1}{N}
        \sum_{j=1}^N
        \norm{\Vb_m}_{\text{op}}
        \Big\{
        \vert 
        \sigma
        \left( \langle \Qb_m \bh_i, \Kb_m \bh_j \rangle + u^1_m \Ub + u^2_m \bar{\Ub} \right)
        \vert
        \cdot \norm{\bh_j - \bh_j^\prime}_2
        \\
        &\quad\quad\quad\quad+
        \vert
        \sigma
        \left( \langle \Qb_m \bh_i, \Kb_m \bh_j \rangle + u^1_m \Ub + u^2_m \bar{\Ub} \right)
        -
        \sigma
        \left( \langle \Qb_m \bh_i^\prime, \Kb_m \bh_j \rangle + u^1_m \Ub + u^2_m \bar{\Ub} \right)
        \vert
        \cdot
        \norm{\bh_j^\prime}_2
        \\
        &\quad\quad\quad\quad+
        \vert
        \sigma
        \left( \langle \Qb_m \bh_i, \Kb_m \bh_j \rangle + u^1_m \Ub + u^2_m \bar{\Ub} \right)
        -
        \sigma
        \left( \langle \Qb_m \bh_i^\prime, \Kb_m \bh_j^\prime \rangle + u^1_m \Ub + u^2_m \bar{\Ub} \right)
        \vert
        \norm{\bh_j^\prime}_2
        \Big\}
        \\
        &\leq
        \sum_{m=1}^M
        \frac{1}{N}
        \sum_{j=1}^N
        \norm{\Vb_m}_{\text{op}}
        \cdot
        3
        \norm{\Qb_m}_{\text{op}}
        \norm{\Kb_m}_{\text{op}}
        \mathtt{R}^2
        \norm{\bh_j - \bh_j^\prime}_2
        \\
        &\leq
        B^3 \mathtt{R}^2 \norm{\bH - \bH^\prime}_{2, \infty}.
    \end{align*}
    Where the third inequality comes from the fact that ReLU is $1$-Lipschitzness, and the fourth and fifth inequality comes from the AM-GM inequality.
    For more details, refer \citep[Section~J.2]{bai2024transformers}
\end{proof}

\begin{corollary}[Lipschitz Constant of Single Layer Moirai Transformer]
    For a fixed number of heads $M$ and hidden dimension $D^\prime$, we consider
    \[
    \Theta_{\text{TF}, 1, B}
    =
    \{
    \bm{\theta}
    =
    ( \bm{\theta}_1, \bm{\theta}_2 )
    \}
    :
    M \text{ heads, }
    \text{hidden dimension }
    D^\prime,
    \norm{\bm{\theta}}_{\text{op}}
    \leq 
    B.
    \]
    Then for the function $\text{TF}^{\mathtt{R}}$ given by
    \[
    \text{TF}^{\mathtt{R}}
    :
    ( \bm{\theta}, \bH )
    \mapsto
    \texttt{clip}_{\mathtt{R}}
    (
    \text{MLP}_{\bm{\theta}_2}
    (
    \text{Attn}_{\bm{\theta}_1}
    (
    \bH
    )
    )),
    \quad
    \bm{\theta} \in 
    \Theta_{\text{TF}, 1, B},
    \bH \in \cH_{\mathtt{R}}.
    \]
    $\text{TF}^{\mathtt{R}}$ is $B_{\Theta}$-Lipschitz w.r.t. $\bm{\theta}$ and $B_H$-Lipschitz w.r.t. $\bH$,
    where 
    $B_{\Theta} = (1 + B^2) ( 1 + \iota ) + B\mathtt{R}(1 + B^3\mathtt{R}^2)$
    and 
    $B_H = (1 + B^2)(1 + B^3 \mathtt{R}^2 )$.
\end{corollary}

\begin{proposition}[Lipschitz Constant of Moirai Transformer]\label{proposition:lipschitz-moirai}
    For a fixed number of heads $M$ and hidden dimension $D^\prime$, we consider
    \[
    \Theta_{\text{TF}, L, B}
    =
    \{
    \bm{\theta}
    =
    ( \bm{\theta}_1^{(1:L)}, \bm{\theta}_2^{(1:L)} )
    \}
    :
    M^{(\ell)} = M,
    D^{(\ell)}
    =
    D^\prime,
    \norm{\bm{\theta}}_{\text{op}}
    \leq 
    B.
    \]
    Then for the function $\text{TF}^{\mathtt{R}}$ is $(L B_H^{L-1} B_{\Theta})$-Lipschitz in $\bm{\theta} \in \Theta_{\text{TF}, L, B}$ for any fixed $\bH$.
\end{proposition}

\begin{proof}
    For any $\bm{\theta} = ( \bm{\theta}_1, \bm{\theta}_2 )$, 
    $\bH \in \cH_{\mathtt{R}}$, and $\theta^\prime = (\theta^\prime_1, \theta^\prime_2)$, we have
    \begin{align*}
        \norm{
        \text{TF}_{\bm{\theta}}(\bH)
        -
        \text{TF}_{\theta^\prime}(\bH)
        }_{2, \infty}
        &\leq
        \norm{
        \text{MLP}_{\bm{\theta}_2}
        (
        \text{Attn}_{\bm{\theta}_1}
        (\bH)
        )
        -
        \text{MLP}_{\bm{\theta}_2}
        (
        \text{Attn}_{\theta^\prime_1}
        (\bH)
        )
        }_{2, \infty}
        +
        \\
        &\quad\quad
        \norm{
        \text{MLP}_{\bm{\theta}_2}
        (
        \text{Attn}_{\theta^\prime_1}
        (\bH)
        )
        -
        \text{MLP}_{\theta^\prime_2}
        (
        \text{Attn}_{\theta^\prime_1}
        (\bH)
        )
        }_{2, \infty}
        \\
        &\leq
        (1 + B^2)
        \norm{
        \text{Attn}_{\theta^\prime_1}
        (\bH)
        -
        \text{Attn}_{\theta^\prime_1}
        (\bH)
        }_{2, \infty}
        +
        B \bar{\mathtt{R}}
        \norm{ \bm{\theta}_2 - \theta_2^\prime }_{\text{op}}
        \\
        &\leq
        (1 + B^2) ( 1 + \iota ) \norm{\bm{\theta}_1 - \theta_1^\prime}_{\text{op}}
        +
        B \bar{\mathtt{R}}
        \norm{\bm{\theta}_2 - \theta_2^\prime}_{\text{op}}
        \leq 
        B_{\Theta} \norm{\bm{\theta} - \theta^\prime}_{\text{op}},
    \end{align*}
    where $\bar{\mathtt{R}} = \mathtt{R} + B^3 \mathtt{R}^3$, $\iota = \max \{  B^2\mathtt{R}^2 +T + (T-1)d, B(T-1)d  \}$.
    The second inequality comes from the fact $\norm{\text{Attn}_{\bm{\theta}}(\bH)} \leq \mathtt{R} + B^3 \mathtt{R}^3$.

    Further, for $\bH^\prime \in \cH_{\mathtt{R}}$, we have
    \begin{align*}
    \norm{
        \text{TF}_{\bm{\theta}}(\bH)  
        -
        \text{TF}_{\bm{\theta}}(\bH^\prime)
    }_{2, \infty}
    &\leq 
    (1 + B^2) \norm{
    \text{Attn}_{\bm{\theta}_1}(\bH)
    -
    \text{Attn}_{\bm{\theta}_1}(\bH^\prime)
    }
    \\
    &\leq
    (1 + B^2)(1 + B^3 \mathtt{R}^2 ) \norm{
    \bH - \bH^\prime
    }_{2, \infty}.
    \end{align*}

    For the multi-layer case, one can simply follow \citep[Proposition~J.1]{bai2024transformers} to conclude the proof.
    
\end{proof}

\clearpage

\subsection{Proof of \cref{thm:gen-bound-1}}\label{proof:gen-bound-1}

Let $\pi$ be a meta distribution, and each distribution drawn from $ \mathtt{P}^{(T)} \sim \pi$ satisfies the Dobrushin's condition.
We then define the single-path average loss as
\[
Y_{\theta, \mathtt{P}^{(T)}}
\coloneqq
\frac{1}{T}
\sum_{t=1}^T
\ell( \theta, \bz_t)
-
\mathbb{E}_{\bz \sim \mathtt{P}^{(T)}}
\left[
\ell( \theta, \bz)
\right].
\]

Now, we assume our pretraining data is generated by the following
\begin{enumerate}
    \item Sample $n$ distributions from $\pi$ i.i.d. to get $\mathtt{P}^{(T)}_j$, for $j = 1,\cdots,n$
    \item For each distribution $\mathtt{P}^{(T)}_j$, we sample $(\bz_{j,1}, \cdots, \bz_{j,T})$
\end{enumerate}

\begin{assumption}\label{assumption:stationary}
    We assume that for each $j \in [n]$, $(z_{j, t})$ has marginals equal to some distribution $D$ for $t = 1, \cdots ,T$.
\end{assumption}

We first present several lemma and theorems that will be used later.
\begin{lemma}[\text{\citep[Example~5.8]{wainwright2019high}}]\label{lem:covering-number-unit-ball}
    Given any well-defined norm $\norm{\cdot}^\prime$.
    Let $\mathbb{B}$ be the $\R^d$ unit-ball in $\norm{\cdot}^\prime$, i.e. 
    $\mathbb{B} = \{ \theta \in \R^d \mid \norm{\theta}^\prime \leq 1 \}$,
    we have
    \[
    \log N( \delta, \mathbb{B}, \norm{\cdot}^\prime )
    \leq
    d 
    \log
    \left( 1 + \frac{2}{\delta} \right).
    \]
\end{lemma}

\begin{theorem}[\text{\citep[Theorem~5.3]{dagan2019learning}}]\label{thm:gen-bound-dobrushin}
    Given a function class $\cF$, such that $\vert f \vert \leq B$, for all $f \in \cF$.
    Let $\mathtt{P}^{(T)}$ be a distribution over some domain $Z^{(T)}$, assuming \cref{assumption:stationary} holds and $\alpha_{\log}( \mathtt{P}^{(T)} ) < \nicefrac{1}{2}$.
    Then for all $t > 0$,
    \[
    P_{\bz \sim \mathtt{P}^{(T)}}
    \left(
    \sup_{f\in\cF}
    \left\vert 
        \frac{1}{T}
        \sum_{i=1}^T
        f(z_i)
        -
        \mathbb{E}_{z}
        [
        f(z)
        ]
    \right\vert
    >
    C
    \left(
        \mathfrak{G}_{\mathtt{P}^{(T)}}( \cF)
        +
        \frac{B t}{\sqrt{T}}
    \right)
    \right)
    \leq
    e^{-t^2/2},
    \]
    for some universal constant whenever $1/2 - \alpha_{\log}(\mathtt{P}^{(T)})$ is bounded away from zero.
\end{theorem}

The following theorem is from \cite{dagan2019learning, kulske2003concentration}.
\begin{theorem}\label{thm:dobrushin-concentration}
    Let $\mathtt{P}^{(T)}_{\bz}$ be a distribution satisfying the Dobrushin's condition with coefficient $\alpha(\mathtt{P}_{\bz}^{(T)})$.
    Let $(\bz_1,\cdots,\bz_T) \sim \mathtt{P}^{(T)}$, and let $f: \bZ^{(T)} \rightarrow \R$ be a real-valued function with the following bounded difference property, with parameters $\lambda_1, \cdots, \lambda_T \geq 0$:
    \[
    \vert f(\bz) - f(\bz^\prime) \vert 
    \leq
    \sum_{t=1}^T
    \mathbbm{1}_{\bz_t\neq \bz_t^\prime}
    \lambda_t.
    \]
    Then for all $t > 0$,
    \[
    P
    \left(
        \vert 
        f(\bz)
        -
        \mathbb{E}
        [ f(\bz) ]
        \geq
        t
    \right)
    \leq
    2 \exp
    \left(
    -
    \frac{(1 - \alpha)t^2}{2 \sum_t \lambda_t^2}.
    \right)
    \]
\end{theorem}
The following corollary directly follows from the above result
\begin{corollary}\label{cor:subgaussian-loss}
    Following \cref{thm:dobrushin-concentration}, let
    \[
    \ell( \bz)
    \coloneqq
    \frac{1}{T}
    \sum_{t=1}^T
    \ell( \bz_t),
    \]
    where $0 \leq \ell(\bz_t) \leq B$ for all $t=1,\cdots,T$ and all $\bz \sim \mathtt{P}_{\bz}$.
    Then the variance of $\ell(\cdot)$ is bounded by
    \[
    \vert 
        \ell(\bz)
        -
        \ell(\bz^\prime)
    \vert
    \leq
    B.
    \]
    Then, the following holds
    \[
        P
        \left(
        \left\vert
        \ell(\bz)
        -
        \mathtt{E}
        [ \ell(\bz) ]
        \right\vert
        \geq
        t
        \right)
        \leq
        2
        \exp
        \left(
        \frac{- (1 - \alpha) t^2 }{2 \sum_t B^2}
        \right)
        .
    \]
\end{corollary}

\paragraph{Direct Application of \cref{thm:gen-bound-dobrushin}.}
By \cref{thm:gen-bound-dobrushin}, if \cref{assumption:stationary} holds, with probability over $1 - e^{-t^2/2}$, for any $\theta \in \Theta$, $\alpha_{\log}(\mathtt{P}^{(T)}_j) < 1/2$ we have

\[
\sup_{\theta \in \Theta }
\left\vert
Y_{\theta, \mathtt{P}^{(T)}}
\right\vert
\leq
C 
\left[
\mathfrak{G}_{\mathtt{P}^{(T)}_j}(\ell(\Theta))
+
\frac{2tB_x^2}{\sqrt{T}}
\right],
\]
where $\ell(\Theta)$ denotes the function class of $\ell( \theta, )$, for all $\theta\in\Theta$, and $C > 0$ is an universal constant.
Note that the above bound presents the naive learning bound for learning a single time series, which is a direct result from \cite{dagan2019learning}.

\begin{proof}

We then define a random process 
$\{ X_\theta \}$ as
\begin{align*}
X_{\theta}
\coloneqq
\frac{1}{n}
\sum^n_{j=1}
Y_{\theta, \mathtt{P}_j^{(T)}}
&=
\frac{1}{n}
\sum^n_{j=1}
\left[
\frac{1}{T}
\sum_{t=1}^T
\ell( \theta, \bz_{j,t})
-
\mathbb{E}_{\bz \sim \mathtt{P}^{(T)}_j}
\left[
\ell( \theta, \bz)
\right]
\right]
\\
&=
\left[
\frac{1}{nT}
\sum^n_{j=1}
\sum_{t=1}^T
\ell( \theta, \bz_{j,t})
\right]
-
\mathbb{E}_{\bz \sim \mathtt{P}^{(T)}_j, \mathtt{P}^{(T)}_j \sim \pi}
\left[
\ell( \theta, \bz)
\right].
\end{align*}

Now, to take supremum over $X_\theta$, we get
\begin{align*}
    \sup_{\theta \in \Theta}
    X_\theta
    &=
    \sup_{\theta \in \Theta}
    \frac{1}{n}
    \sum_{j=1}^n
    Y_{\theta, \mathtt{P}_j^{(T)}}
    \\
    &\leq
    \frac{1}{n}
    \sum_{j=1}^n
    \sup_{\theta \in \Theta}
    Y_{\theta, \mathtt{P}_j^{(T)}}.
\end{align*}

To upper bound $\sup | X_{\theta} |$,   we take a similar approach to \citep[Proposition~A.4]{bai2024transformers}.

Assuming the index set $\Theta$ is equipped with a distance metric $\rho$ and diameter $D$.
We assume that for any ball $\Theta^\prime$ of radius $r$ in $\Theta$, there exists some constant $C_1$ such that the covering number admits upper bound
\[
\log
N( \delta, \Theta^\prime, \rho)
\leq
d \log ( 2 A r / \delta),
\]
for all $0 < \delta \leq 2r$.

Now we select $\Theta_0$ such that it is a $(D_0/2)$-covering of $\Theta$.
The above assumption guarantees us that we can have a $\Theta_0$ such that $\log | \Theta_0 | \leq d \log (2AD/D_0)$.
By \cref{cor:subgaussian-loss}, $X_{\bm{\theta}}$ is a $\nicefrac{2 B_x^2}{(1 - \alpha)}$-subgaussian ($\alpha = \alpha( \mathtt{P}^{(T)})$).
Then, with probability at least $1 - \delta/2$, 
\[
\sup_{\theta \in \Theta_0}
| X_\theta |
\leq
C \frac{2 B_x^2}{(1 - \alpha)}
\sqrt{
d \log (2 A D/ D_0)
+
\log ( 2 / \delta)
}.
\]

Note that the uniform bound for independent subgaussian random variables still applies here as for each $\theta$, we are re-sampling a new chain from a new distribution sampled from $\pi$.

Assume that $\Theta_0 = \{ \theta_1, \cdots, \theta_n\}$.
Now for each $j\in[m]$, we consider $\Theta_j$ is the ball centered at $\theta_j$ of radius $D_0$ in $(\Theta, \rho)$.
With \cref{thm:generalizd-dudley}, for each process $\{ X_{\theta} \}_{\theta \in\Theta_j}$, then
\[
\psi = \psi_2,
\quad
\norm{
X_{\theta}
-
X_{\theta^\prime}
}_{\psi}
\leq
\frac{B^1}{\sqrt{n}}
\rho(\theta, \theta^\prime),
\]
where $\ell(\theta, \bz) - \ell(\theta^\prime, \bz)$ is a $B^1 \rho(\theta, \theta^\prime)$-subgaussian random variable.

We then get
\[
P
\left(
    \sup_{\theta, \theta^\prime \in \Theta_j}
    \vert 
    X_{\theta}
    -
    X_{\theta^\prime}
    \vert 
    \leq
    C^\prime 
    B^1
    D_0
    \left(
    \sqrt{
    \frac{d\log(2A)}{n}
    }
    +t
    \right)
\right)
\leq
2 \exp( -nt^2),
\quad\text{ for all }t \geq 0.
\]

If we further take $t \leq \sqrt{\log(2m/\delta)/n}$, then with probability at least $1 - \delta/2$, it holds that for all $j\in[m]$,
\[
    \sup_{\theta, \theta^\prime \in \Theta_j}
    \vert 
    X_{\theta}
    -
    X_{\theta^\prime}
    \vert 
    \leq
    C^\prime
    B^1
    D_0
    \sqrt{
    \frac{
    2d \log (2AD/D_0) + \log(4/\delta)
    }{n}
    }.
\]

By chaining, we have
\[
|X_{\theta}|
\leq 
|X_{\theta_j}|
+
|X_{\theta}
-
X_{\theta_j}|.
\]

Hence with probability at least $1 - \delta$, it holds that

\[
\sup_{\theta\in\Theta}
|X_{\theta}|
\leq
\sup_{\theta\in\Theta_0} | X_{\theta} |
+
\sup_j
\sup_{\theta \in \Theta_j}
| X_{\theta} - X_{\theta_j} |
\leq 
C^{\prime\prime}
(
\frac{2 B_x^2}{(1 - \alpha)}
+
B^1 D_0
)
\sqrt{
\frac{
d \log (2AD/D_0) + \log(2/\delta)
}{n}
}.
\]

Next by taking $D_0 = D/\kappa, \kappa= 1 + B^1D \frac{(1 - \alpha)}{2 B_x^2}$, we get

\[
\sup_{\theta\in\Theta}
|X_{\theta}|
\leq
C^{\prime\prime}
(
\frac{2 B_x^2}{(1 - \alpha)}
+
B^1 D \kappa
)
\sqrt{
\frac{
d \log (2A \kappa) + \log(2/\delta)
}{n}
}.
\]

Last, we check whether the assumptions we make above hold for our function class $\ell_\Theta$. 
Below, we slightly abuse our notation by using $D$ as the dimension for weight matrices in $\text{TF}_{\bm{\theta}}$.
By \cref{lem:covering-number-unit-ball}, it holds that 
\[
\log N (\delta, B_{\norm{\cdot}_{\text{op}}}(r), \norm{\cdot}_{\text{op}} )
\leq
L(3MD^2 + D D^\prime + 2) \log ( 1 + 2r/\delta),
\]
where $B_{\norm{\cdot}_{\text{op}}}(r)$ is a ball of radius $r$ under norm $\norm{\cdot}_{\text{op}}$.

We check that
\[
\norm{
\ell(\theta, \bz)
-
\ell(\theta^\prime, \bz)
}
\leq
B_x(L B_H^{L-1} B_{\Theta}) \norm{\theta - \theta^\prime}_{\text{op}},
\]
where it is a direct result from \cref{proposition:lipschitz-moirai}.
By plugging all the parameters, we get

\[
\sup_{\theta\in\Theta}
|X_{\theta}|
\leq
C
(
\frac{B_x^2}{(1 - \alpha)}
)
\sqrt{
\frac{
L(3MD^2 + D D^\prime + 2) \iota + \log(2/\delta)
}{n}
},
\]

where $\iota = \log( 2 + 2(L B_H^{L-1} B_{\Theta}) B \frac{1-\alpha}{B_x})$

Finally, by plugging the ERM $\hat{\bm{\theta}}$, we get
\[
L(\hat{\bm{\theta}})
\leq
\inf_{\theta} L(\theta)
+
2 \sup_{\theta} | X_{\theta} |.
\]

\end{proof}

\clearpage

\subsection{Analysis of Section~\ref{sec:AR1}}\label{appendix:analysis-ar1}

\begin{definition}[Markov Random Field (MRF) with pairwise potentials]
The random vector $\mathcal{Z} = (\cZ_1, \cdots, \cZ_d)$ over $Z^d$ is an MRF with pariwise potentials if there exist functions $\psi_i : Z \rightarrow \R$ and $\varphi_{ij}: Z^2 \rightarrow \R$ for $i \neq j \in \{ 1, \cdots, d \}$ such that for all $z \in Z^d$,
\begin{equation*}
    \mathbb{P}_{z \sim \mathtt{P}^d }
    \left[ 
    \cZ = z
    \right]
    =
    \prod_{i=1}^d
    e^{\psi(\cZ_i)}
    \prod_{1 \leq i < j \leq d}
    e^{\varphi_{ij}(\cZ_i, \cZ_j)}
\end{equation*}
The functions $\psi_i$ are called as element-wise potentials and $\phi_{ij}$ are pairwise potentials.  
\end{definition}

\begin{definition}
Given an MRF $\cZ$ with potentials $\{ \varphi_i \}$ and $\{ \psi_{ij}\}$, we define
    \begin{equation*}
        \beta_{i,j}(\cZ)
        \coloneqq
        \sup_{\cZ_i, \cZ_j \in Z}
        \lvert
        \varphi_{ij}( \cZ_i, \cZ_j )
        \rvert;
        \quad
        \beta(\cZ)
        \coloneqq
        \max_{1 \leq i \leq d}
        \sum_{j \neq i}
        \beta_{ij}
        ( \mathtt{P}^d ).
    \end{equation*}
\end{definition}

\begin{lemma}\label{lem:beta-bound}
    Given an MRF $\bz$ with pairwise potentials, for any $i \neq j$, $I_{j \rightarrow i}(\bz) \leq \beta_{j,i}(\bz)$.
    $ \bI_{j\rightarrow i}(\cZ)
    \leq
    \bI^{\log}_{j,i}(\cZ) 
    \leq 
    \beta_{j,i}(\cZ)$
    
\end{lemma}
\cref{lem:beta-bound} implies that to satisfy the condition $\alpha^{\log}(\cdot) < 1/2$, it is sufficient to show that $\beta(\cdot) < 1/2$, leading to the following condition.
\begin{equation}
    \langle \bw \bx_{t}, \bx_{t+1} \rangle < \ln \frac{1}{2} + ( \sigma_{\epsilon}^2 ).
\end{equation}

Observe that
\begin{align*}
        \langle \bw \bx_{t}, \bx_{t+1} \rangle 
        &\leq
        \norm{\bw} \cdot \max_t \norm{\bx_t}
        \\
        &= B_w B_x
        \\&< \ln \frac{1}{2} + ( \sigma_{\epsilon}^2 )
        \sim 0.3
        .
\end{align*}

\subsection{Additional Details}

\paragraph{The History Matrix.}
The matrix form of $\mathtt{A}_i(q)$ is presented below
\begin{equation}\label{eqn:history-matrix}
    \mathtt{A}_i(q)
    \coloneqq
    \begin{bmatrix}
        x_1^i & x_2^i & \cdots & x_t^i & x_{t+1}^i & x_{t+2}^i & \cdots \\
        x_T^i & x_{T-1}^i & \cdots & x_{t-1}^i & x_t^i & x_{t+1}^i & \cdots \\
        x_{T-1}^i & x_{T-2}^i & \cdots & x_{t-2}^i & x_{t-1}^i & x_{t}^i & \cdots \\
        \vdots & \vdots & \vdots & \vdots & \vdots & \vdots & \cdots \\
        x_{T-q}^i & x_{T-q+1}^i & \cdots & x_{t-q}^i & x_{t-q+1}^i & x_{t-q+2}^i & \cdots \\   
\end{bmatrix}
\end{equation}

\clearpage

\section{Experimental Details}\label{sec:exp-details}

\subsection{Environment}
We mostly train our model on NVIDIA-H100 GPUs with 2 cores each with 128GB RAM.
2 GPUs are sufficient for all of our experiments.
We use PyTorch 2.1 and our code is based on the open source published by \cite{woo2024unified}.
Training and evaluate takes roughly $12$ hours for one run.

\subsection{Model Architecture}
For most of our experiments, we use MOIRAI-base model.
The hyperparameters are listed in \cref{table:hyperparameters}.

\begin{table}[h]
        \centering
        \caption{Hyperparameters
        }
        \begin{tabular}{l*{1}{c}}
        \toprule
            \bf{parameter} & \multicolumn{1}{c}{\bf{values}}  \\ 
            \midrule
            batch size & $64$ \\ 
            initial learning rate
            & $1\text{e-}3$ \\
            learning rate decay
            & cosine annealing \\
            hidden dimension
            & $768$ \\
            MLP dimension
            & $3072$ \\
            number of heads
            & $12$\\
            training steps
            & $20k$ \\
            max sequence length & $512$ \\
            \midrule
            optimizer
            & AdamW \\
            beta ($\beta_1, \beta_2$)
            & $(0.9, 0.98)$ \\
            weight decay & $1\text{e-}1$ \\
            warm up steps (linear) & $10k$ \\
            \bottomrule
        \end{tabular}
    \label{table:hyperparameters}
\end{table}

\subsection{Synthetic Data Generation}
We generate the $\mathtt{AR}$ synthetic data similar to Equation~\eqref{eqn:AR-data} but use normalization to stabilize the values.
The parameters of synthetic data are in \cref{table:data-param}.
Consider a sequence of data 
$\bx \in \R^{d \times T} \coloneqq (\bx_1, \dots, \bx_T)$, where $\bx_t = (x_t^1, \cdots, x_t^d) \in \R^d$.
Assuming our target (variate of interest) is in dimension $1$, we assume the $\mathtt{AR}_d(q)$ process generates $x_t^1$ as follows:
\begin{equation}
    x_{t}^1
    =
    \frac{1}{qd}
    \sum_{i=1}^q
    \sum_{j=1}^d
    a_i^j \cdot x_{t-i}^j
    + \epsilon_t
    ,
\end{equation}
where $\epsilon_t \sim N(0, 1)$, $a_i^j \sim N(0, 1) \in \R$.
After recursively generating the time series, we remove its first 50 time steps as burnout.
Each $\mathtt{AR}$ time series has a number of covariates between $1$ to $5$.
For training data, we sampled $100$ different time series, each with $20k$ time steps.
For test data, we randomly generate one time series with time step $5k$, and evaluate our model on all time steps.
We set $q, d \leq 5$ in our experiments.

\paragraph{Seasonality.}
We also conduct experiments on datasets with seasonality information.
Specifically, we consider monthly information.
After generating a multi-variate time series with $T$ time steps $\bx \in \R^{d \times T}$, we then add the seasonality information.
For each time step $t$, its seasonal information is
\[
a \cdot \sin{ \frac{2 \pi T}{f}} \in \R,
\]
where $a \in \R$ is the amplitude, $f \in \N^+$ is the frequency which is $30$ for monthly information.
The whole seasonal information will be added to the time series.

\begin{table}[h]
        \centering
        \caption{Parameter of Synthetic Data
        }
        \begin{tabular}{l*{1}{c}}
        \toprule
            \bf{parameter} & \multicolumn{1}{c}{\bf{values}}  \\ 
            \midrule
            lag size & $\{ 1, 2, 3, 4, 5 \}$ \\ 
            variance
            &  $\text{unif}(0.1, 1)$ \\
            length ($T$)
            & $20k$ \\
            number of covariates ($d$)
            & $\{1, 2, 3, 4, 5 \}$ \\
            \midrule
            amplitude &  $\text{unif}(0, 1.5)$ \\
            frequency &  $30$ \\
            \bottomrule
        \end{tabular}
    \label{table:data-param}
\end{table}

\subsection{Baselines}

\paragraph{Least Squares Regression.}\label{appendix:ls-baseline}
Consider MOIRAI taking an input AR sequence $\bx \in \R^{d \times T}$, to match our theoretical results (\cref{thm:any-variate-auto}), we transform $\bx$ into the following input-label pairs
\begin{align*}
    \tilde{\bx}_1 &= 
    \left(
    ( \bx_1, \cdots \bx_q ),
    \bx_{q+1}
    \right)
    \\
    \tilde{\bx}_2 &= 
    \left(
    ( \bx_2, \cdots \bx_{q+1} ),
    \bx_{q+2}
    \right)...
\end{align*}
After fitting least squares on this transformed dataset with $T-q$ samples, it predicts the $T+1$-th time step with the following input
\[
\tilde{\bx}_{\text{test}} = 
\left(
\bx_{T-q+1}, \cdots \bx_{T}
\right).
\]
For least squares, we use learning rate as $0.1$, and perform full gradient descent with $50, 100$ iterations.

\subsection{Additional Experiments}\label{appendix:additional-exp}

\paragraph{Seasonality Data.}
Here we present the experimental results on training transformers on seasonality data.
The data generation is the same as described above.
We use the same setup for seasonality data, where our training data comes from time series with $d \in \{ 1, 2, 3, 4, 5 \}$, and $q = \{ 1, 2, 3, 4, 5\}$.
The evaluation results on seasonality data is in \cref{fig:seasonal}.
We observe that transformers are capable of inferring data with seasonality.
Note that transformers are capable of achieving nearly optimal performance, while least squares regression fails, indicating that transformers are capable of fitting a more complicated model than $\mathtt{AR}$ on a given time series.

\begin{figure}[h]
    \centering
    \includegraphics[width=0.5\linewidth]{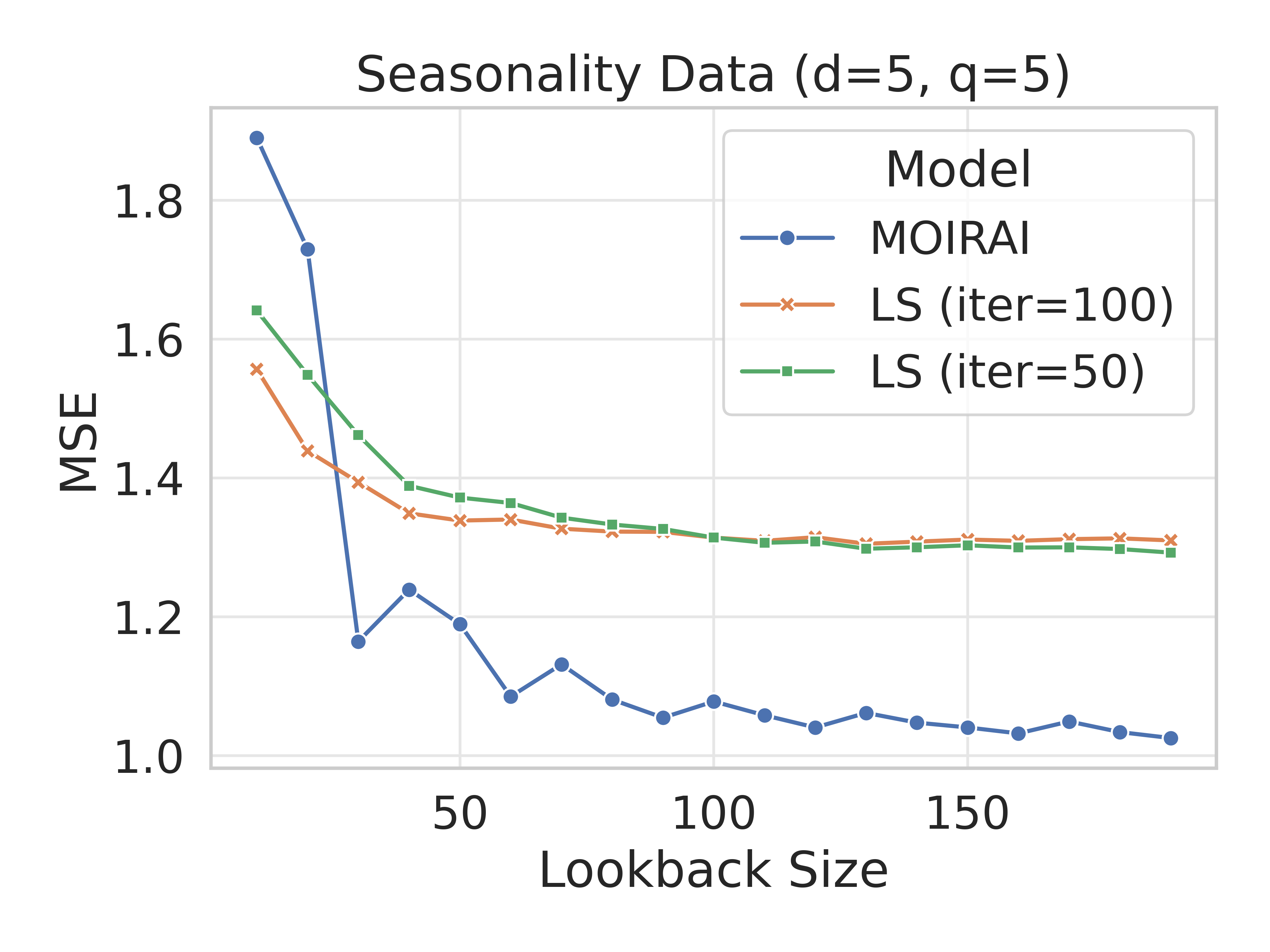}
    \caption{We observe that when least squares regression fails to obtain the optimal error rate for prediction, transformers are capable of having their MSE converge towards $1$ as the lookback size increases.
    This indicates that these models are capable of fitting a more complex model other than linear regression on a given time series.
    }
    \label{fig:seasonal}
\end{figure}




\end{document}